%% file: main.tex
\documentclass[11pt]{article} 

\usepackage{times}
\usepackage{amsfonts}
\usepackage{xcolor}
\usepackage{amsthm}

\usepackage{amsmath}
\usepackage{amssymb}
\usepackage[numbers,sort&compress]{natbib}

\usepackage{graphicx}
\usepackage{url}
\usepackage{algorithm2e}
\PassOptionsToPackage{algo2e,ruled}{algorithm2e}
\usepackage{amsthm}
\usepackage[margin=1.1in]{geometry}
\usepackage{hyperref}
\hypersetup{colorlinks=true,linkcolor=blue,citecolor=blue}

\usepackage{times}
\usepackage[mathscr]{eucal}

\newtheorem{theorem}{Theorem}
\newtheorem{lemma}[theorem]{Lemma}
  
  \newtheorem{proposition}[theorem]{Proposition}
  \newtheorem{remark}[theorem]{Remark}
  \newtheorem{corollary}[theorem]{Corollary}
  \newtheorem{definition}[theorem]{Definition}

\newcommand{\comment}[1]{}

\title{Quadratic Memory is Necessary for Optimal Query Complexity in Convex Optimization: Center-of-Mass is Pareto-Optimal}
\usepackage{times}

\author{
  Mo\"ise Blanchard\\
  MIT\\
  \small{\texttt{moiseb@mit.edu}}
  \and
  Junhui Zhang\\
  MIT\\
  \small{\texttt{junhuiz@mit.edu}}
  \and 
  Patrick Jaillet\\
  MIT\\
  \small{\texttt{jaillet@mit.edu}}
}
\date{}
\setcitestyle{numbers,open={[},close={]}} 

\usepackage{thmtools}
\usepackage[mathscr]{eucal}
\usepackage{bbm}

\PassOptionsToPackage{algo2e,ruled,linesnumbered}{algorithm2e}

\newenvironment{game}[1][]
  {\renewcommand{\algorithmcfname}{Game}%
   \begin{algorithm}[#1]
   \long\def\@caption##1[##2]##3{%
     \par
     \begingroup\@parboxrestore
     \if@minipage\@setminipage\fi
     \normalsize \@makecaption{\AlCapSty{\AlCapFnt\algorithmcfname}}
     \par\endgroup
   }}
  {\end{algorithm}
}

\newenvironment{proc}[1][]
  {\renewcommand{\algorithmcfname}{Procedure}%
   \begin{algorithm}[#1]
   \long\def\@caption##1[##2]##3{%
     \par
     \begingroup\@parboxrestore
     \if@minipage\@setminipage\fi
     \normalsize \@makecaption{\AlCapSty{\AlCapFnt\algorithmcfname}}{\ignorespaces ##3}%
     \par\endgroup
   }}
  {\end{algorithm}
}

\newcommand{\acks}[1]{\section*{Acknowledgments}#1}

\newcommand{\nonl}{\renewcommand{\nl}{\let\nl\oldnl}}

\DeclareMathOperator*{\argmax}{arg\,max}

\DeclareMathOperator*{\polylog}{polylog}

\usepackage{latexsym,tikz,graphicx}
\usetikzlibrary{calc}
\usetikzlibrary{decorations.pathreplacing}
\usetikzlibrary{patterns}

\renewenvironment{proof}[1][]{\par\noindent{\bf Proof #1\ }}{\hfill$\blacksquare$\\[2mm]}

\begin{document}

\include{shortcuts}
\maketitle

\begin{abstract}%
    We give query complexity lower bounds for convex optimization and the related feasibility problem. We show that quadratic memory is necessary to achieve the optimal oracle complexity for first-order convex optimization. In particular, this shows that center-of-mass cutting-planes algorithms in dimension $d$ which use $\tilde\Ocal(d^2)$ memory and $\tilde\Ocal(d)$ queries are Pareto-optimal for both convex optimization and the feasibility problem, up to logarithmic factors. Precisely, building upon techniques introduced in \cite{marsden2022efficient}, we prove that to minimize $1$-Lipschitz convex functions over the unit ball to $1/d^4$ accuracy, any deterministic first-order algorithms using at most $d^{2-\delta}$ bits of memory must make $\tilde\Omega(d^{1+\delta/3})$ queries, for any $\delta\in[0,1]$. For the feasibility problem, in which an algorithm only has access to a separation oracle, we show a stronger trade-off: for at most $d^{2-\delta}$ memory, the number of queries required is $\tilde\Omega(d^{1+\delta})$. This resolves a COLT 2019 open problem of Woodworth and Srebro.
\end{abstract}


\section{Introduction}\label{sec:introduction}

We consider the canonical problem of first-order convex optimization in which one aims to minimize a convex function $f:\Rbb^d\to \Rbb$ with access to an oracle that for any query $\mb x$ returns $(f(\mb x),\nabla f(\mb x))$ the value of the function and a subgradient of $f$ at $\mb x$. Arguably, this is one of the most fundamental problems in optimization, mathematical programming and machine learning.

A classical question is how many oracle queries are required to guarantee finding an $\epsilon$-approximate minimizer for any $1$-Lipschitz convex functions $f:\Rbb^d\to\Rbb$ over the unit ball. We denote by $B_d(\mb x, r)=\{\mb x'\in \Rbb^d:\|\mb x-\mb x'\|_2\leq \epsilon\}$ the ball centered in $\mb x$ of radius $r$. There exist methods that given first-order oracle access only need $\Ocal(d\log 1/\epsilon)$ queries and this query complexity is worst-case optimal \cite{nemirovskij1983problem} when $\epsilon\ll 1/\sqrt d$. Known methods achieving the optimal $\Ocal(d\log 1/\epsilon)$ query complexity fall in the broad class of cutting plane methods, that build upon the well-known ellipsoid method \cite{yudin1976informational, shor1977cut} which uses $\Ocal(d^2\log 1/\epsilon)$ queries. These include the inscribed ellipsoid \cite{tarasov1988method,nesterov1989self}, volumetric center or Vaidya's method \cite{atkinson1995cutting,vaidya1996new}, approximate center-of-mass via sampling techniques \cite{levin1965algorithm,bertsimas2004solving} and recent improvements \cite{lee2015faster,jiang2020improved}. Unfortunately, all these methods suffer from at least $\Omega(d^3\log 1/\epsilon)$ time complexity and further require storing all subgradients, or at least an ellipsoid in $\Rbb^d$, therefore at least $\Omega(d^2\log 1/\epsilon)$ bits of memory. These limitations are prohibitive for large-scale optimization, hence cutting plane methods are viewed as rather impractical and less frequently used for high-dimensional applications. On the other hand, the simplest, perhaps most commonly used and practical gradient descent requires $\Ocal(1/\epsilon^2)$ queries, which is not optimal for $\epsilon\ll 1/\sqrt d$, but only needs $\Ocal(d)$ time per query and $\Ocal(d\log 1/\epsilon)$ memory.

A natural question is whether one can preserve the optimal query lower bounds from cutting-planes methods with simpler methods, for instance, inspired by gradient descent techniques. Such hope is largely motivated by the fact that in many different theoretical settings, cutting plane methods have achieved state-of-the-art runtimes including semidefinite programming \cite{anstreicher2000volumetric,lee2015faster}, submodular optimization \cite{mccormick2005submodular,grotschel2012geometric,lee2015faster,jiang2021minimizing} or equilibrium computation \cite{papadimitriou2008computing,jiang2011polynomial}. Towards this goal, \cite{woodworth2019open} first posed this question in terms of query complexity / memory trade-off: given a certain number of bits of memory, which query complexity is achievable?  While cutting planes methods require $\Omega(d^2\log 1/\epsilon)$ memory, gradient descent only requires storing one vector and as a result, uses $\Ocal(d\log 1/\epsilon)$ memory, which is information-theoretically optimal \cite{woodworth2019open}\footnote{$\Omega(d\log 1/\epsilon)$ bits of memory are already required just to represent the answer to the optimization problem.}. Understanding this trade-off could pave the way for the design of more efficient methods in convex optimization. 

The first result in this direction was provided in \cite{marsden2022efficient}, where they showed that it is impossible to be both optimal in query complexity and in memory. Specifically, they proved that any potentially randomized algorithm that uses at most $d^{1.25-\delta}$ memory must make at least $\tilde\Omega(d^{1+4/3\delta})$ queries. This implies that a super-linear amount of memory $d^{1.25}$ is required to achieve the optimal rate of convergence (that is achieved by algorithms using more than quadratic memory). However, this leaves open the fundamental question of whether one can improve over the memory of cutting-plane methods while keeping optimal query complexity.

\paragraph{Question (COLT 2019 \cite{woodworth2019open}).} Is it possible for a first-order algorithm that uses at most $\Ocal(d^{2-\delta})$ bits of memory to achieve query complexity $\tilde\Ocal(d \polylog 1/\epsilon)$ when $d=\Omega(\log^c 1/\epsilon)$ but $d=o(1/\epsilon^c)$ for all $c>0$?\\

In this paper, building upon the techniques introduced in \cite{marsden2022efficient}, we provide a negative answer to this question: quadratic memory is necessary to achieve the optimal query complexity with deterministic algorithms. As a result, cutting plane methods including the standard center-of-mass algorithm are Pareto-optimal up to logarithmic factors within the query complexity / memory trade-off. Our main result for convex optimization is the following.

\begin{theorem}\label{thm:main_opt}
    For $\epsilon =1/d^4$ and any $\delta\in[0,1]$, a deterministic first-order algorithm guaranteed to minimize $1$-Lipschitz convex functions over the unit ball with $\epsilon$ accuracy uses at least $d^{2-\delta}$ bits or makes $\tilde\Omega(d^{1+\delta/3})$ queries.
\end{theorem}

A key component of cutting plane methods is that they merely rely on the subgradient information at each query to restrict the search space. As a result, these can be used to solve the larger class of feasibility problems that are essential in mathematical programming and optimization. In a feasibility problem, one aims to find an $\epsilon$-approximation of an unknown vector $\mb x^\star$, and has access to a separation oracle. For any query $\mb x$, the separation oracle either returns a separating hyperplane $\mb g$ from $\mb x$ to $B_d(\mb x^\star,\epsilon)$---such that $\langle \mb g,\mb x-\mb z\rangle>0$ for any $\mb z\in B_d(\mb x^\star,\epsilon)$---or signals that $\|\mb x -\mb x^\star\|\leq \epsilon$. This class of problems is broader than convex optimization since the negative subgradient always provides a separating hyperplane from a suboptimal query to the optimal set. Hence, feasibility and convex minimization problem are closely related and it is often the case that obtaining query lower bounds for the feasibility problem simplifies the analysis while still providing key insights for the more restrictive convex optimization problem \cite{nemirovskij1983problem,nesterov2003introductory}.

As a result, a similar fundamental question is to understand the query complexity / memory trade-off for the feasibility problem. As noted above, any lower bound for convex optimization yields the same lower bound for the feasibility problem. Here, we can significantly improve over the previous trade-off.

\begin{theorem}\label{thm:main_feasibility}
For $\epsilon=1/(48d^2\sqrt d)$ and any $\delta\in[0,1]$, a deterministic algorithm guaranteed to solve the feasibility problem over the unit ball with $\epsilon$ accuracy uses at least $d^{2-\delta}$ bits of memory or makes at least $\tilde \Omega(d^{1+\delta})$ queries.
\end{theorem}

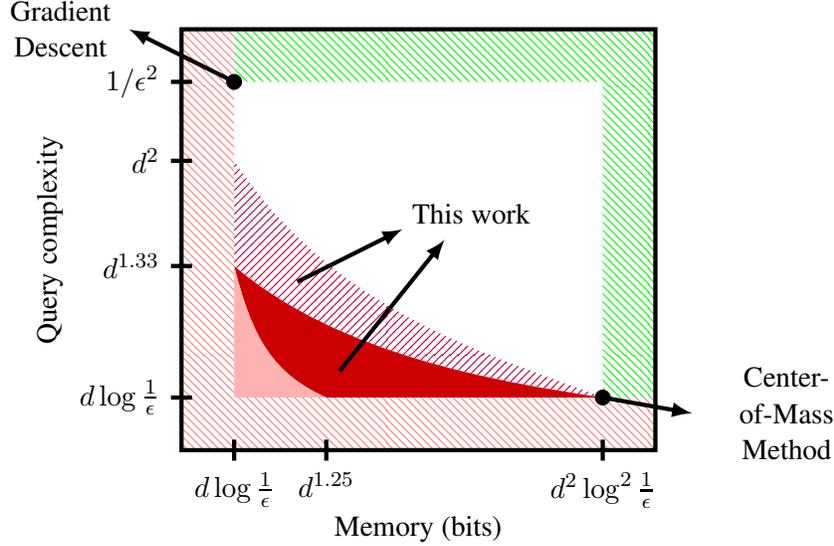
\begin{figure}[ht]
    \centering
    \begin{tikzpicture}[scale=0.7]
    
    \draw[thick,draw=none,pattern=north west lines,pattern color=red!50] (-1,-1) -- (-1,7) -- (0,7) -- (0,0) -- (8,0) -- (8,-1);
    
    \draw[thick,draw=none,pattern=north west lines,pattern color=green] (0,7) -- (0,6)-- (7,6)--(7,0) -- (8,0) -- (8,7);
    
    
    \draw[ultra thick,black] (-1,-1) rectangle (8,7);

    \draw[draw=none,pattern=north east lines,pattern color=purple] (0,4.5) .. controls (2,1.5) and (5,0.8) .. (7,0) .. controls  (5,0.3) and (2,0.8) .. (0,2.5) ;
    
    \draw[draw=none,fill=red!80!black] (0,2.5) .. controls (0.35,1.1) and (0.75,0.45) .. (7/4,0) -- (7,0) .. controls  (5,0.3) and (2,0.8) .. (0,2.5) ;
    \draw[draw=none,fill=red!30] (0,0) -- (0,2.5) .. controls (0.35,1.1) and (0.75,0.45) .. (7/4,0) ;

    \draw[ultra thick,black] (0,-1.2) node[below]{$d\log \frac{1}{\epsilon}$}  -- (0,-0.8);

    \draw[ultra thick,black] (7/4,-1.2) node[below]{$d^{1.25}$}  -- (7/4,-0.8);
    \draw[ultra thick,black] (7,-1.2) node[below]{$d^2\log^2\frac{1}{\epsilon}$}  -- (7,-0.8);

    \draw[ultra thick,black] (-1.2,0) node[left]{$d\log\frac{1}{\epsilon}$}  -- (-0.8,0);
    \draw[ultra thick,black] (-1.2,2.5) node[left]{$d^{1.33}$}  -- (-0.8,2.5);
    \draw[ultra thick,black] (-1.2,4.5) node[left]{$d^{2}$}  -- (-0.8,4.5);
    \draw[ultra thick,black] (-1.2,6) node[left]{$1/\epsilon^2$}  -- (-0.8,6);

    \draw (3.5,-2.5) node{Memory (bits)};
    \draw (-3.5,3) node[rotate=90]{Query complexity};

    \filldraw (0,6) circle (4pt);
    \draw [ultra thick,black,-latex] (0,6) -- (-2,7) node[left,align=center,text width=1.5cm]{Gradient Descent} ;

    \draw [ultra thick,black,-latex] (7,0) -- (8.7,-0.3)node [align=center,right,text width=2.2cm]{Center-of-Mass Method};
    \filldraw (7,0) circle (4pt);

    \draw[ultra thick,-latex] (2,0.5) -- (4,3);
    \draw[ultra thick,-latex] (1.2,2.2) -- (3.2,3.2);
    \draw (4.5,3.5) node{This work};
    
    \end{tikzpicture}
    
    \caption{Trade-offs between available memory and first-order oracle complexity for minimizing 1-Lipschitz convex functions over the unit ball (adapted from \cite{woodworth2019open,marsden2022efficient}). The dashed pink “L” (resp. green inverted "L") shaped region corresponds to historical information-theoretic lower bounds (resp. upper bounds) on the memory and query-complexity. The solid pink region corresponds to the recent lower bound trade-off from \cite{marsden2022efficient}, which holds for randomized algorithms. In our work, we show that the solid red region is not achievable for any deterministic algorithms. For the feasibility problem, we also show that the dashed red region is not achievable either for any deterministic algorithms.}
    \label{fig:memory_query_complexity_tradeoff}
\end{figure}

\subsection{Literature review}
Recently, there has been a series of studies exploring the trade-offs between sample complexity and memory constraints for learning problems, such as linear regression \cite{Steinhardt15,Sharan2019}, principal component analysis (PCA) \cite{Mitliagkas2013}, learning under the statistical query model \cite{steinhardt16} and other general learning problems \cite{Brown2021, brown22a,Moshkovitz2017,Moshkovitz2018,Beame2018,Garg2018,Kol2017}.

For parity problems that meet certain spectral (mixing) requirements, \cite{Raz2017} first proved by a computation tree argument that an exponential number of random samples is needed if the memory is sub-quadratic. Similar trade-offs have been obtained when the learning problem satisfies other types of properties \cite{Moshkovitz2017,Moshkovitz2018,Beame2018,Garg2018,Kol2017}. It should be noted that all the above-mentioned results hold for learning problems over finite fields, i.e. the concept classes are finite. For continuous problems, \cite{Sharan2019} was the first to apply \cite{Raz2017}'s framework and showed a sample-complexity lower bound for memory-constrained linear regression.

In contrast to learning with random samples, there is limited understanding of the memory-constrained optimization and feasibility problem. \cite{Nemirovsky1983} demonstrated that, in the absence of memory constraints, finding an $\epsilon$-approximate solution for Lipschitz convex functions requires $\Omega(d\log 1/\epsilon)$ queries, which can be achieved by the center-of-mass method using $O(d^2\log^21/\epsilon)$ bits of memory. At the other extreme, gradient descent needs $\Omega(1/\epsilon^2)$ queries but only $O(d\log1/\epsilon)$ bits of memory, the minimum memory needed to represent a solution. These two extreme cases are represented by dashed pink ``impossible region'' and dashed green ``achievable region'' in Figure \ref{fig:memory_query_complexity_tradeoff}. Since then, \cite{marsden2022efficient} showed that there is a trade-off between memory and query for convex optimization: it is impossible to be both optimal in query complexity and memory. Their lower bound is represented by the solid pink ``impossible region'' in Figure \ref{fig:memory_query_complexity_tradeoff}. In this paper, we significantly improve these results to match the quadratic upper bound of cutting plane methods. Additionally, there has been recent progress in the study of query complexity for randomized algorithms \cite{Woodworth2016, Woodworth2017LowerBF}. 

On the algorithmic side, the afore-mentioned methods that achieve $O(poly(d))$ query complexity \cite{yudin1976informational, shor1977cut,tarasov1988method,nesterov1989self,atkinson1995cutting,vaidya1996new,levin1965algorithm,bertsimas2004solving,lee2015faster,jiang2020improved} all require at least $\Omega(d^2\log1/\epsilon)$ bits of memory. There is also significant literature on memory-efficient optimization algorithms, such as the Limited-memory-BFGS \cite{Nocedal1980,liu_limited_1989}. However, the convergence behavior for even the original BFGS on non-smooth convex objectives is still a challenging, open question \cite{lewis_nonsmooth_2013}.



\paragraph{Comparison with \cite{marsden2022efficient}}

Our proof techniques build upon those introduced in \cite{marsden2022efficient}. We follow the proof strategy that they introduced to derive lower bounds for the memory/query complexity. Below, we delineate which ideas and techniques are borrowed from \cite{marsden2022efficient} and which are the novel elements that we introduce. Details on these proof elements are given in Section \ref{subsec:proof_techniques}. 

First, \cite{marsden2022efficient} define a class of difficult functions for convex optimization of the following form
\begin{equation}\label{eq:previous_form}
    \max\left\{\|\mb A\mb x\|_\infty - \eta_0,\eta_1\left(\max_{i\leq N} \mb v_i^\top \mb x -i\gamma \right) \right\},
\end{equation}
where $\mb A\sim\Ucal(\{\pm 1\}^{d/2\times d})$ is a matrix with $\pm 1$ entries sampled uniformly, and $v_i\sim\Ucal(d^{-1/2}\{\pm 1\}^d)$ are sampled independently, uniformly within the rescaled hypercube. To give intuition on this class, the term $\|\mb A\mb x\|_\infty-\eta_0$ acts as barrier : in order to observe subgradients from the other term, one needs to use queries $\mb x$ that are approximately within the nullspace of $\mb A$. The second term $\max_{i\leq N} \mb v_{i}^\top \mb x - i\gamma$ is the ``Nemirovski'' function, which was used in previous works \cite{nemirovski1994parallel,balkanski2018parallelization,bubeck2019complexity} to obtain lower bounds in parallel convex optimization. At a high level, the limitation in the lower bounds from \cite{marsden2022efficient} comes from the fact that one is limited in the number $N$ of vectors $\mb v_1,\ldots,\mb v_N$ that can be used in the Nemirovski function. To resolve this issue, we introduce adaptivity within the choice of a modified Nemirovski function. At a high level, we choose the vectors $\mb v_1,\ldots,\mb v_N$ depending on the queries of the algorithm which allows to fit in more terms. In turn, this allows to improve the lower bounds.

As a second step, \cite{marsden2022efficient} relate the optimization problem on the defined class of functions to an Orthogonal Vector Game. In this game, the goal is to find vectors that are approximately orthogonal to a matrix $\mb A$ with access to row queries of $\mb A$. The argument is as follows: because of the barrier term $\|\mb A\mb x\|_\infty -\eta_0$, optimizing the Nemirovski function requires exploring independent directions of the nullspace of $\mb A$, which is performed at \emph{informative queries}. With our new class of functions, we can adapt this logic. However, the adaptivity in the vectors $\mb v_i$ provides information to the learner on $\mb A$ in addition to the queried rows of $\mb A$. We therefore need to modify the game by introducing an Orthogonal Vector Game with Hints, where hints encapsulate this extra information.

For the last step,  \cite{marsden2022efficient} give an information-theoretic argument to provide a query complexity lower bound on the defined Orthogonal Vector Game. Following the same structure, we show that a similar argument holds for our modified game. The main added difficulty resides in bounding the information leakage from the hints, and we show that these provide no more information than the memory itself.

As a last remark, the lower bounds provided in \cite{marsden2022efficient} hold for randomized algorithms, while the adaptivity of our procedure only applies to deterministic algorithms.

\subsection{Outline of paper}
Our main results for the trade-off between memory and query complexity for optimization and feasibility problem have been presented in Section \ref{sec:introduction} (Theorem \ref{thm:main_opt}, \ref{thm:main_feasibility}). In Section \ref{sec:setup_statement}, we formally define memory-constrained algorithms and provide a brief overview of our proof techniques and contributions. Our proofs for convex optimization are given in Section \ref{sec:optimization}. We introduce the \textit{optimization procedure} which adaptively constructs a hard family of functions, provide a reduction from this hard family to an \textit{orthogonal vector game with hints}, and show a memory-sample trade-off (Proposition \ref{prop:lower_bound_queries_game}) for the game, which completes the proof of the Theorem \ref{thm:main_opt}. Last, in Section \ref{sec:feasibility}, we consider the feasibility problem and, with a similar methodology, prove Theorem \ref{thm:main_feasibility}.

\section{Formal setup and overview of techniques}\label{sec:setup_statement}

Standard results in oracle complexity give the minimal number of queries for algorithms to solve a given problem. However, this does not account for possible restrictions on the memory available to the algorithm. In this paper, we are interested in the trade-off between memory and query complexity for both convex optimization and the feasibility problem. Our results apply to a large class of \emph{memory-constrained} algorithms. We give below a general definition of the memory constraint for algorithms with access to an oracle $\Ocal:\Scal\to\Rcal$ taking as input a query $q\in\Scal$ and returning as response $\Ocal(q)\in\Rcal$.

\begin{definition}[$M$-bit memory-constrained deterministic algorithm]
\label{def:memory_constraint}
    Let $\Ocal:\Scal\to \Rcal$ be an oracle. An $M$-bit memory-constrained deterministic algorithm is specified by a query function $\psi_{query}:\{0,1\}^M\to \Scal$ and an update function $\psi_{update}:\{0,1\}^M \times \Scal \times \Rcal \to \{0,1\}^M$. The algorithm starts with the memory state $\mathsf{Memory}_0 = 0^M$ and iteratively makes queries to the oracle. At iteration $t$, it makes the query $q_t=\psi_{query}(\mathsf{Memory}_{t-1})$ to the oracle, receives the response $r_t=\Ocal(q_t)$ then
    updates its memory $\mathsf{Memory}_t = \psi_{update}(\mathsf{Memory}_{t-1},q_t, r_t)$.
\end{definition}

The algorithm can stop making queries at any iteration and the last query is its final output.
Notice that the memory constraint applies only between each query but not for internal computations, i.e. the computation of the update $\psi_{update}$ and the query $\psi_{query}$ can potentially use unlimited memory. This is a rather weak memory constraint on the algorithm; a fortiori, our negative results also apply to stronger notions of memory-constrained algorithms. In Definition \ref{def:memory_constraint}, we ask the query and update functions to be time-invariant. In our context, this is without loss of generality: any $M$-bit algorithm using $T$ queries with time-dependent query and update functions \cite{woodworth2019open,marsden2022efficient} can be turned into an $(M+\lceil\log T\rceil)$-bit time-invariant algorithm by storing the iteration number $t$ as part of the memory. The query lower bounds we provide are at most $T\leq poly(d)$. Hence, an additional $\log T= O(\log d)$ bits to the memory size $M$, does not affect our main results, Theorems \ref{thm:main_opt} and \ref{thm:main_feasibility}. 

In this paper, we use the above described framework to study the interplay between query complexity and memory for two fundamental problems in optimization and machine learning.

\paragraph{Convex optimization.} We first consider convex optimization in which one aims to minimize a $1$-Lipschitz convex function $f:\Rbb^d\to \Rbb$ over the unit ball $B_d(0,1)\subset \Rbb^d$. The goal is to output a point $\tilde{\mb x}\in B_d(0,1)$ such that $f(\tilde{\mb x})\leq \min_{\mb x\in B_d(0,1)} f(\mb x) + \epsilon$, referred to as $\epsilon$-approximate points. The optimization algorithm has access to a first order oracle $\Ocal_{CO}:\Rbb^d\to\Rbb\times\Rbb^d$, which for any query $\mb x$ returns the couple $(f(\mb x), \partial f(\mb x))$ where $\partial f(\mb x)$ is a subgradient of $f$ at the query point $\mb x$. 

\begin{remark}
The above requirement for $\epsilon$-approximate optimality is weaker than asking to find a point that is at distance $\epsilon$ from $\arg\min_{\mb x\in B_d(\mb 0,1)}f(\mb x)$ (for $1$-Lipschitz convex functions). As a result, our lower bounds for $\epsilon$-approximate optimality hold a fortiori for the problem where one aims to find a point at distance at most $\epsilon$ from the solution set.
\end{remark}

\paragraph{Feasibility problem.} Second, we consider the trade-off between memory and query complexity for the feasibility problem, where the goal is to find an element $\tilde {\mb x}\in Q$ for a convex set $Q\subset B_d(0,1)$. Instead of a first-order oracle, the algorithm has access to a separation oracle $\Ocal_F:\Rbb^d\to \{\mathsf{Success}\}\cup \Rbb^d$. For any query $\mb x\in \Rbb^d$, the separation oracle either returns $\mathsf{Success}$ reporting that $\mb x\in Q$, or provides a separating vector $\mb g\in\Rbb^d$, i.e., such that for all $\mb x'\in Q$,
    \begin{equation*}
        \langle \mb g,\mb x-\mb x'\rangle > 0.
    \end{equation*}
We say that an algorithm solves the feasibility problem with accuracy $\epsilon>0$ if it can solve any feasibility problem for which the successful set contains a ball of radius $\epsilon$, i.e., such that there exists $\mb x^\star\in B_d(0,1)$ satisfying $B_d(\mb x^\star,\epsilon)\subset Q$. 

\comment{

Analogous to the definition for the $M$-bit memory-constrained deterministic first-order algorithm with $T$ queries, we have the following definition for the feasibility problem algorithm:
\begin{definition}[$M$-bit memory-constrained deterministic feasibility algorithm with $T$ queries]
    A deterministic feasibility algorithm is said to use $M$ bits memory and $T$ queries if it's specified by an update function $\psi_{update}:\{0,1\}^M \times \Rbb^d \times \Rbb^d\to \{0,1\}^M$ and a query function $\psi_{query}:\{0,1\}^M \to \Rbb^d $. The algorithm starts with a memory state $\mathsf{Memory}_0\in \{0,1\}^M$. At $t = 1,2,\ldots,T$, it makes the query $\mb x_t = \psi_{query}(\mathsf{Memory}_{t-1})$ to the separation oracle: if the oracle returns $\mb x_t\in Q$, the algorithm outputs $\mb x_t$ and then terminates, otherwise the oracle returns a separation hyperplane $\mb g_t$, and the algorithm updates $\mathsf{Memory}_t = \psi_{update}(\mathsf{Memory}_{t-1},\mb x_t, \mb g_t)$. If by iteration $T$, there is no output yet, the final output of the algorithm is $\psi_{query}(\mathsf{Memory}_T)$.
    \end{definition}

Our main result for the feasibility problem is the following.
\begin{theorem}\label{thm:main_theorem_feasibility}
For $\epsilon=1/(48d^2\sqrt d)$ and any $\delta\in[0,1]$, an algorithm guaranteed to solve the feasibility problem with optimality $\epsilon$ uses at least $d^{2-\delta}$ bits of memory or makes $\tilde \Omega(d^{1+\delta})$ queries.
\end{theorem}

}

The feasibility problem is at least as hard as convex optimization in the following sense: an algorithm that solves the feasibility problem with accuracy $\epsilon/L$ can be used to solve $L$-Lipschitz convex optimization problems by feeding the subgradients from first-order queries to the algorithm as separating hyperplanes. Alternatively, from any $1$-Lipschitz function $f$ one can derive a feasibility problem, where the feasibility set is $Q = \{\mb x\in B_d(0,1) , f(\mb x) \leq f^\star + \epsilon\}$ and the separating oracle at $\mb x \notin Q$ is a subgradient $\partial f(\mb x)$ at $\mb x$.

\subsection{Overview of proof techniques and innovations}\label{subsec:proof_techniques}

We prove the two main Theorems \ref{thm:main_opt} and \ref{thm:main_feasibility} with similar techniques, hence for conciseness, we only give here the main ideas used to derive lower bounds for convex optimization. Although our proof borrows its structure and techniques from \cite{marsden2022efficient}, we introduce key innovations involving adaptivity to improve the lower bounds up to the maximum quadratic memory for deterministic algorithms---up to logarithmic factors. We recall, however, that the bounds in \cite{marsden2022efficient} hold for randomized algorithms as well. In the proofs, we aim to optimize the dependence of the parameters in $d$. Constants, however, are not necessarily optimized.

\paragraph{An adaptive optimization procedure.} At the high level, we design an \emph{optimization procedure} which for any algorithm constructs a hard family of convex functions adaptively on its queries. To be precise, the procedure constructs functions from the following family of convex functions with appropriately chosen parameters $\eta, \gamma_1,\gamma_2,p_{max},l_p,\delta$:

\begin{equation}\label{eq:form_bad_functions}
    F_{\mb A, \mb v} (\mb x) =  \max\left\{\|\mb A\mb x\|_\infty - \eta,\eta \mb v_0^\top \mb x, \eta\left(\max_{p\leq p_{max},l\leq l_p} \mb v_{p,l}^\top \mb x - p\gamma_1 - l\gamma_2\right)\right\}.
\end{equation}
 We take $\mb A \sim \mathcal U(\{\pm 1\}^{n \times d})$ and $\mb v_0 \sim \mathcal U( \mathcal D_{\delta})$ uniformly sampled in the beginning, where $\mathcal D_{\delta}\subset \mathcal S^{d-1}$ is a (finite) discretization of the sphere. The first term $\|\mb A\mb x\|_\infty - \eta$ acts as a barrier term: in order to observe subgradients from the other terms, one needs the query $\mb x$ to satisfy $\|\mb A\mb x\|_\infty \leq 2\eta$. These are called \emph{informative queries} as introduced in \cite{marsden2022efficient}. Hence, informative queries must lie approximately in the orthogonal space to the lines of $\mb A$. The second term $\eta\mb v_0^\top \mb x$ is used to ensure that solutions with low objective (in particular with objective at most $\eta\gamma_1/2$) have norm bounded away from $0$. As a result, these informative queries, once renormalized, will still belong approximately to the nullspace of $\mb A$ denoted $Ker(\mb A)$.
 
 The adaptivity to the algorithm is captured in the third term, which is constructed along the optimization process. This construction proceeds by periods $p=1,2,\ldots,p_{max}$ designed so that during each period $p$, the algorithm is forced to visit a subspace of $Ker(\mb A)$ of dimension $k$. To do so, we iteratively construct vectors $\mb v_{p,1},\ldots \mb v_{p,l_p}$ as follows. Suppose that at the beginning of step $t$ of period $p$, one has defined vectors $\mb v_{p,1},\ldots,\mb v_{p,l}$.
 \begin{itemize}
     \item The procedure first evaluates the explored subspace of the algorithm during this period. In practice, the procedure keeps in memory \emph{exploratory} queries $\mb x_{i_{p,1}},\ldots, \mb x_{i_{p,r}} $ during period $p$ up to time $t$. The exploratory subspace is then $Span(\mb x_{i_{p,1}},\ldots, \mb x_{i_{p,r}})$.
     \item If a query with a sufficiently low objective is queried, we sample a new vector $\mb v_{p,l+1}$ which is approximately orthogonal to the exploratory subspace. The corresponding new term in the objective is $\mb v_{p,l+1}^\top \mb x - p\gamma_1 - (l+1)\gamma_2$.
 \end{itemize}
 Once this new term is added to the objective, the algorithm is constrained to make queries with an additional component along the direction $-\mb v_{p,l+1}$. Since this vector is approximately orthogonal to all previous queries, this forces the algorithm to query vectors linearly independent from all previous queries in period $p$. The period then ends once the dimension of the exploratory subspace reaches $k$, having defined $l_p$ vectors $\mb v_{p,1},\ldots,\mb v_{p,l_p}$. As discussed above, the exploratory subspace must increase dimension for any additional such vector. Thus, after $l_p\leq k$ vectors, period $p$ ends.

The constructed family of convex functions in Eq~\eqref{eq:form_bad_functions} is similar to the family described in Eq~\eqref{eq:previous_form} that were considered in \cite{marsden2022efficient}. However, by sampling the vectors $\mb v_{p,l}$ adaptively, the \textit{optimization procedure} is able to fit in more terms, thereby providing a significant improvement in the lower bounds.

\comment{
 The constructed family of convex functions is similar to the family considered in \cite{marsden2022efficient},
 \begin{equation}\label{eq:previous_form}
     \max\left\{\|\mb A\mb x\|_\infty - \eta_0, \eta_1\left(\max_{i\leq N} \mb v_{i}^\top \mb x - i\gamma\right)\right\},
 \end{equation}
 which combines the wall term $\|\mb A\mb x\|_\infty -\eta_0$ with a rescaled ``Nemirovski'' function $\max_{i\leq N} \mb v_{i}^\top \mb x - i\gamma$, where the vectors $\mb v_1,\ldots,\mb v_N$ are sampled independently from each other and from the algorithm's queries. The Nemirovski function was used in previous works \cite{nemirovski1994parallel,balkanski2018parallelization,bubeck2019complexity} to obtain lower bounds in parallel convex optimization. While the form of our constructed functions \eqref{eq:form_bad_functions} is similar, by sampling the vectors $\mb v_{p,l}$ adaptively, the \textit{optimization procedure} is able to fit in more terms, thereby providing a significant improvement in the lower bounds.
 }

 \paragraph{Benefits of adaptivity.} We now expand on how the adaptive terms allow improving the lower bound of \cite{marsden2022efficient} to match the quadratic upper bound of cutting plane methods. The limitation in the functions of the form Eq~\eqref{eq:previous_form} comes from the fact that the offset in the Nemirovski function is $\gamma=\Omega(\sqrt{k\log d/d})$. This offset is necessary to ensure that with high probability, 1. subgradients $\mb v_1,\ldots,\mb v_N$ are discovered exactly in this order and 2. that any query which visits a new vector $\mb v_i$ must not lie in the subspace formed by the last $k$ last informative vectors. Indeed, for the last claim, from high-dimensional concentration, for a random unit vector $\mb v$ and a $k$ dimensional subspace $E$, $\|P_E(\mb v)\|=\Theta(\sqrt{k\log d/d})$. This offset is not necessary for our procedure, since by construction, at each period, a $k$-dimensional subspace of $Ker(\mb A)$ is forced to be explored. As a result, we can take $\gamma_1=\Theta(\sqrt{\log d/d})$. This offset is still necessary to ensure that vectors $\mb v_{p,l}$ are discovered in their order of construction (lexicographic order on $(p,l)$) with high probability.

\paragraph{An Orthogonal Vector Game with Hints.}

The next step of the proof involves linking the optimization of the above-mentioned constructed functions with an Orthogonal Vector Game with Hints. Similarly to the game introduced by \cite{marsden2022efficient}, the goal for the player is to find $k$ linearly-independent vectors approximatively in $Ker(\mb A)$. To do so, the player can access an $M$-bit message $\mathsf{Message}$ and make $m$ queries, where $M=ckd$ for a small constant $c>0$. In the game introduced by \cite{marsden2022efficient}, the queries are lines of the matrix $\mb A$. They then show that to find $k$ dimensions of $\mb A$, where $\mb A$ is taken uniformly at random $\mb A\sim\{\pm 1\}^{d/2\times d}$, (nearly) all the lines of $\mb A$ must be queried. The argument is information-theoretic: each new dimension of $Ker(\mb A)$ must be (approximately) orthogonal to all lines of $\mb A$. Hence, this provides additional mutual information $O(k)$ for every line of $\mb A$, including the $d/2-m$ lines that were not observed through queries. This extra information on $\mb A$ can only be explained by the message, which has $M$ bits. Hence, $M \geq O(k)(d/2-m)$. Setting the constant $c>0$ appropriately, this shows that $m=\Omega(d)$.

In our case, the optimization procedure ensures that the algorithm needs to explore $k$ dimensions of $Ker(\mb A)$ in each period. However, each query yields a response from the optimization oracle that can either be a line of $\mb A$ (corresponding to the term $\|\mb A\mb x\|_\infty -\eta$ of Eq~\eqref{eq:form_bad_functions}) or $\mb v_0$ (term $\eta\mb v_0^\top\mb x$ of Eq~\eqref{eq:form_bad_functions}), or previously defined vectors $\mb v_{p,'l,'}$. Now since the vectors $\mb v_{p',l'}$ have been constructed adaptively on the queries of the algorithm, which themselves may depend on lines of $\mb A$, during a period $p$, responses $\mb v_{p',l'}$ for $p'<p$ are a source of information leakage for $\mb A$ from previous periods. As a result, the query lower bound on the game introduced by \cite{marsden2022efficient} is not sufficient for our purposes. Instead, we introduce an Orthogonal Vector Game with Hints, where hints correspond exactly to these vectors $\mb v_{p',l'}$ from previous periods. Informally, the game corresponds to a simulation of one of the periods of the optimization procedure: for each query $\mb x$, the oracle returns the subgradient that would have been returned in the optimization procedure, up to minor details.

\paragraph{Bounding the information leakage.}

Once the link is settled, the goal is to prove lower bounds on the number of queries needed to solve the Orthogonal Vector Game with Hints. The main difficulty is to bound the information leakage from these hints. We recall that hints are of the form $\mb v_{p',l'}$, which have been constructed adaptively on the queries of the algorithm during period $p'$. In particular, these contain information on the lines of $\mb A$ queried during period $p'<p$, which may be complementary with those queried during period $p$. If this total information leakage through the hints yields a mutual information with $Ker(\mb A)$ significantly higher than that of the $M$ bits of $\mathsf{Message}$, obtained lower bounds cannot possibly reflect any trade-off with memory constraints. It is therefore essential to obtain information leakage at most $\Ocal(M)=\tilde \Ocal(dk)$.

To solve this issue, we introduce a discretization $\Dcal_\delta$ of the unit sphere where the vectors $\mb v_{p,l}$ take value. Next, we show that each individual vector $\mb v_{p',l'}$ from previous periods can only provide information $\tilde\Ocal(k)$ on the matrix $\mb A$. To have an intuition on this, note that for any (at most) $k$ vectors $\mb x_1,\ldots,\mb x_k$, the volume of the subset of the unit sphere $S^{d-1}$ of vectors approximately orthogonal to $\mb x_1,\ldots,\mb x_k$, say $S(\mb x_1,\ldots,\mb x_k) =\{\mb y\in S^{d-1}:|\mb y^\top \mb x_i|\leq d^{-3},i\leq k\}$ is $q_k=\Ocal(1/d^{3k})$. Hence, since the vector $\mb v$ is roughly taken uniformly at random within $\Dcal_\delta\cap S(\mb x_1,\ldots,\mb x_k)$, we can show that the mutual information of $\mb v$ with the initial vectors $\mb x_1,\ldots,\mb x_k$ is at most $\Ocal(-\log q_k) = \Ocal(k\log d)$. As a result, even if $m=d$, the total information leakage through the vectors $\mb v_{p',l'}$ from previous periods, is at most $\Ocal(kd\log d)$. The formal proof involves an anti-concentration bounds on the distance of a random unit vector to a linear subspace of dimension $k$, as well as a more involved discretization procedure than the one presented above. In summary, by introducing adaptive functions through the optimization procedure, we show that the same memory-sample trade-off holds for the Orthogonal Vector Game with Hints and the game without hints introduced in \cite{marsden2022efficient}, up to logarithmic factors.

\section{Memory-constrained convex optimization}
\label{sec:optimization}

To prove our results we need to use discretizations of the unit sphere $S^{d-1}$. It will be convenient to ensure that the partitions induced by these discretizations have equal area, which can be done with the following lemma.

\begin{lemma}[\cite{feige2002optimality} Lemma 21]
\label{lemma:feige}
For any $0<\delta<\pi/2$, the sphere $S^{d-1}$ can be partitioned into $N(\delta) =(\Ocal(1)/\delta)^d$ equal volume cells, each of diameter at most $\delta$.
\end{lemma}

We denote by $\Vcal_\delta = \{V_i(\delta), i\in[N(\delta)]\}$ the corresponding partition, and consider a set of representatives $\Dcal_\delta = \{\mb b_i(\delta), i\in [N(\delta)]\}\subset S^{d-1}$ such that for all $i\in [N(\delta)]$, $\mb b_i(\delta)\in V_i(\delta)$. With these notations we can define the discretization function $\phi_\delta$ as follows
\begin{equation*}
    \phi_\delta(\mb x) = \mb b_i(\delta) ,\quad \mb x\in V_i(\delta).
\end{equation*}

\subsection{Definition of the difficult class of optimization problems}

In this section we present the class of functions that we use to prove our lower bounds. Throughout the paper, we pose $n=\lceil d/4 \rceil$. We first define some useful functions. For any $\mb A\in \Rbb^{n\times d}$, we define $\mb g_{\mb A}$ as follows
\begin{equation*}
    \mb g_{\mb A}(\mb x) = \mb a_{i_{\min}},\qquad i_{\min} = \min\{i\in[n], |\mb a_i^\top \mb x|= \|\mb A\mb x\|_\infty\}.
\end{equation*}
With this function we can define a subgradient function for $\mb x\mapsto \|\mb A\mb x\|_\infty$,
\begin{equation*}
    \tilde{\mb g}_{\mb A}(\mb x) = \epsilon \mb g_{\mb A}(\mb x),\qquad \epsilon = sign(\mb g_{\mb A}(\mb x)^\top \mb x).
\end{equation*}

We are now ready to introduce the class of functions which we use for our lower bounds. These are of the following form.
\begin{equation*}
    F_{\mb A, \mb v} (\mb x) =  \max\left\{\|\mb A\mb x\|_\infty - \eta,\eta \mb v_0^\top \mb x, \eta\left(\max_{p\leq p_{max}}\max_{l\leq l_p} \mb v_{p,l}^\top \mb x - p\gamma_1 - l\gamma_2\right)\right\}.
\end{equation*}
Here, $\mb A\in\{\pm 1\}^{n\times d}$ is a matrix. Also, $\mb v_0$ and the terms $\mb v_{p,l}$ are vectors in $\Rbb^d$. More precisely, these vectors will lie in the discretization $\Dcal_\delta$ for $\delta=1/d^3$. We postpone the definition of $p_{max}$ and $l_p$ for $p\leq p_{max}$. Last, we use the following choice for the remaining parameters: $\eta= 2/d^3$, $\gamma_1 = 12\sqrt{\frac{\log d}{ d}}$  and $\gamma_2 = \frac{\gamma_1}{4d}$. For convenience, we also define the functions
\begin{align*}
    F_{\mb A} (\mb x) &= \max\{\|\mb A\mb x\|_\infty - \eta ,\eta \mb v_0^\top \mb x\} \\
    F_{\mb A, \mb v, p ,l} (\mb x) &=  \max\left\{\|\mb A\mb x\|_\infty - \eta ,\eta \mb v_0^\top \mb x, \eta\left(\max_{(p',l')\leq_{lex} (p,l), l'\leq l_{p
    '}} \mb v_{p',l'}^\top \mb x - p'\gamma_1 - l'\gamma_2\right)  \right\},
\end{align*}
with the convention $F_{\mb A, \mb v, 1,0} = F_{\mb A}$. The functions $F_{\mb A, \mb v, p, l}$ will encapsulate the current state of the function to be minimized: it will be updated adaptively on the queries of the algorithm. We also define a subgradient function for $ F_{\mb A, \mb v, p ,l}$ by first favoring lines of $\mb A$, then vectors from $\mb v$ in case of ties, as follows,
\begin{equation*}
    \partial F_{\mb A, \mb v, p,l}(\mb x) = \begin{cases}
        \tilde{\mb g}_{\mb A}(\mb x_t) &\text{if }F_{\mb A, \mb v, l ,p}(\mb x) = \|\mb A\mb x\|_\infty-\eta,\\
        \eta \mb v_0 &\text{otherwise and if }F_{\mb A, \mb v, l ,p}(\mb x) = \eta \mb v_0^\top \mb x,\\
        \eta \mb v_{p,l} &\text{otherwise and if } (p,l) = \argmax_{(p',l')\leq_{lex} (p,l)} \mb v_{p',l'}^\top \mb x - p'\gamma_1 - l'\gamma_2.
    \end{cases}
\end{equation*}
In the last case, ties are broken by lexicographic order. We define $\partial F_{\mb A, \mb v} = \partial F_{\mb A, \mb v, p_{max}, l_{p_{max}}}$ similarly.

We consider a so-called \emph{optimization procedure}, which will construct the sequence of vectors $\mb v=(\mb v_{p,l})$ adaptively on the responses of the considered algorithm. Throughout this section, we use a parameter $1 \leq k \leq d/3-1$ --- which will be taken as $k=\tilde\Theta(M/d)$ where $M$ is the memory of the algorithm --- and let $p_{max}$ be the largest number which satisfies the following constraint.
\begin{equation}\label{eq:definition_p_max}
    p_{max} \leq \min\{(c_{d,1}d-1)/k,\, c_{d,2}(d/k)^{1/3} -1\},
\end{equation}
where $c_{d,1} = 1/(90^2\log^2 d)$ and $c_{d,2} = 1/(81\log^{2/3}d)$.

\begin{proc}[ht]
        
\caption{The optimization procedure for algorithm $alg$}\label{proc:optimization}

\setcounter{AlgoLine}{0}
\SetAlgoLined
\LinesNumbered

\everypar={\nl}

\hrule height\algoheightrule\kern3pt\relax
\KwIn{$d$, $k$, $p_{max}$, algorithm $alg$}
\vspace{5pt}

{\nonl \textbf{Part 1:}} Procedure to adaptively construct $\mb v$\;

Sample $\mb A\sim\Ucal(\{\pm 1\}^{n\times d})$ and $\mb v_0\sim\Ucal(\Dcal_\delta)$.\;

Initialize the memory of $alg$ to $\mb 0$ and let $p=1$, $r=l=0$.\;

\For{$t\geq 1$}{
    \lIf{$t>d^2$}{
        Set $(P,L)=(p,l)$ and break the \textbf{for} loop
    }

    Run $alg$ with current memory to obtain a query $\mb x_t$\;
        
    \uIf(\tcp*[h]{Non-informative query}){$F_{\mb A}(\mb x)> \eta$}{
        \Return $(\|\mb A\mb x_t\|_\infty -\eta,\tilde{\mb g}_{\mb A}(\mb x_t))$ as response to $alg$.
    }
    \uElse(\tcp*[h]{Informative query}){
        \uIf{$r\leq k-1$ and $F_{\mb A,\mb v, p, l}(\mb x_t) \leq -\eta\gamma_1/2$ and $\|P_{Span(\mb x_{i_{p,r'}},r'\leq r)^\perp}(\mb x_t)\| / \|\mb x_t\| \geq \frac{\gamma_2}{4}$}{
            Set $i_{p,r+1}=t$ and increment $r\leftarrow r+1$.
        }
        \uIf{$F_{\mb A, \mb v, p,l}(\mb x_t) < -\eta(p\gamma_1 + l \gamma_2+\gamma_2/2)$ and $r<k$}{
            Compute Gram-Schmidt decomposition $\mb b_{p,1},\ldots,\mb b_{p,r}$ of $\mb x_{i_{p,1}},\ldots, \mb x_{i_{p,r}}$.\;


            Sample $\mb y_{p,l+1}$ uniformly on $\mathcal S^{d-1} \cap \{\mb z\in \Rbb^d : |\mb b_{p,r'}^\top \mb z| \leq d^{-3},\forall r'\leq r\}$.\;

            Define $\mb v_{p,l+1}=\phi_\delta(\mb y_{p,l+1})$ and increment $l\leftarrow l+1$.
        }
        \uElseIf{$F_{\mb A, \mb v, p,l}(\mb x_t) < -\eta(p\gamma_1 + l \gamma_2+\gamma_2/2)$ and $p+1\leq p_{max}$}{
            Set $l_p = l$ and $i_{p+1,1}=t$.\;

            Compute the Gram-Schmidt decomposition $\mb b_{p+1,1}$ of $\mb x_{i_{p+1,1}}$.\;


            Sample $\mb y_{p+1,1}$ uniformly on $\mathcal S^{d-1} \cap \{ \mb z\in \Rbb^d : |\mb b_{p+1,1}^\top \mb z| \leq d^{-3}\}$.\;

            Define $\mb v_{p+1,1}=\phi_\delta(\mb y_{p+1,1})$, increment $p\leftarrow p+1$ and reset $l=r=1$.
        }
        \uElseIf(\tcp*[h]{End of the construction}){$F_{\mb A, \mb v, p,l}(\mb x_t) < -\eta(p\gamma_1 + l \gamma_2+\gamma_2/2)$}{
            Set $l_{p_{max}} = l$, $i_{p_{max}+1,1}=t$.\;
            
            Set $(P,L)=(p_{max},l)$ and break the \textbf{for} loop.
        }
        \Return $(F_{\mb A, \mb v, p, l}(\mb x_t),\mb \partial F_{\mb A, \mb v, p, l}(\mb x_t))$ as response to $alg$.
    }
}

\vspace{5pt}

{\nonl \textbf{Part 2:}} Procedure once $\mb v$, $P$, $L$ are constructed\;

\lFor{$t'\geq t$}{
    \Return $(F_{\mb A, \mb v, P, L}(\mb x_{t'}),\partial F_{\mb A, \mb v, P, L}(\mb x_{t'}) )$ as response to the query $\mb x_{t'}$
}

\hrule height\algoheightrule\kern3pt\relax
\end{proc}

The optimization procedure is described in Procedure \ref{proc:optimization}.
First, we sample independently $\mb A\sim\Ucal(\{\pm 1\}^{n\times d})$ and $\mb v_0\sim\Ucal(\Dcal_\delta)$. The matrix $\mb A$ and vector $\mb v_0$ are then fixed for the rest of the learning procedure. Next, we describe the adaptive procedure to return subgradients. It proceeds by periods, until $p_{max}$ periods are completed, unless the total number of iterations reaches $d^2$, in which case the construction procedure ends as well. First, we say that a query is informative if $F_{\mb A}(\mb x) \leq \eta$. The procedure proceeds by periods $p\in [p_{max}]$ and in each period constructs the vectors $\mb v_{p,1},\ldots,\mb v_{p,k}$ iteratively. We are now ready to describe the procedure at time $t$ when the new query $\mb x_t$ is queried. Let $p\geq 1$ be the index of the current period and $\mb v_{p,1},\ldots, \mb v_{p,l}$ be the vectors of this period constructed so far: the first period is $p=1$ and we allow $l=0$ here. As will be seen in the construction, we always have $l\geq 1$ except at the very beginning for which we use the notation $F_{\mb A, \mb v, 1, 0}= F_{\mb A}$. Together with these vectors, the oracle keeps in memory indices $i_{p,1},\ldots,i_{p,r}$ with $r\leq k$ of \emph{exploratory} queries. The constructed vectors from previous periods are $\mb v_{p',l'}$ for $p'<p$ and $l'\leq l_{p'}$.
\begin{enumerate}
    \item If $\mb x_t$ is not informative, i.e. $F_{\mb A}(\mb x) > \eta$, then procedure returns $(\|\mb A\mb x_t\|_\infty-\eta,\tilde{\mb g}_{\mb A}(\mb x_t))$.
    \item Otherwise, we follow the next steps. If $r\leq k-1$ and
    \begin{equation*}
        F_{\mb A, \mb v, p, l}(\mb x_t) \leq -\frac{\eta\gamma_1}{2} \qquad \text{and}\qquad \frac{\|P_{Span(\mb x_{i_{p,r'}},r'\leq r)^\perp}(\mb x_t)\|}{\|\mb x_t\|} \geq \frac{\gamma_2}{4},
    \end{equation*}
    we set $i_{p,r+1}=t$ and increment $r$. In this case, we say that $\mb x_t$ is \emph{exploratory}. Next,
    \begin{enumerate}
       \item Recalling that $F_{\mb A,\mb v,p,l}$ is constructed so far, if $F_{\mb A,\mb v,p,l}(\mb x_t) \geq \eta(-p\gamma_1 - l \gamma_2-\gamma_2/2)$, we do not do anything.
    \item Otherwise, and if $r<k$, let $\mb b_{p,1},\ldots,\mb b_{p,r}$ be the result from the Gram-Schmidt decomposition of $\mb x_{i_{p,1}},\ldots, \mb x_{i_{p,r}}$. Then, let $\mb y_{p,l+1}$ be a sample of the distribution obtained by the uniform distribution $\mb y_{p,l+1} \sim \Ucal(S^{d-1}\cap \left\{\mb z\in \Rbb^d : |\mb b_{p,r'}^\top \mb z| \leq \frac{1}{d^3},\forall r'\leq r\right\})$. We then pose $\mb v_{p,l+1} = \phi_\delta(\mb y_{p,l+1})$. Having defined this new vector, we increment $l$.
    \item Otherwise, if $r=k$, this ends period $p$. We write the total number of vectors defined during period $p$ as $l_p := l$. If $p+1\leq p_{max}$, period $p+1$ starts from $t=i_{p+1,1}$. Similarly to above, let $\mb b_{p+1,1}$ be the result of the Gram-Schmidt procedure on $\mb x_{p+1,1}$, and we sample $\mb y_{p+1,1}$ according to a uniform distribution $\mb y_{p+1,1}\sim\Ucal(S^{d-1}\cap \left\{ \mb z \in \Rbb^d :  |\mb b_{p+1,1}^\top \mb z| \leq \frac{1}{d^3}\right\})$. Then, we pose $\mb v_{p+1,1} = \phi_\delta(\mb y_{p+1,1})$. We can then increment $p$ and reset $l=r=1$.
    \end{enumerate}
    After these steps, with the current values of $p$ and $l$, we return $(F_{\mb A,\mb v,p,l}(\mb x_t), \partial F_{\mb A, \mb v, l, p}(\mb x_t)) $.
\end{enumerate}
If we finished the last period $p=p_{max}$, or if we reached a total number of iterations $d^2$, the construction phase of the function ends. In both cases, let us denote by $P,L$ the last defined period and vector $\mb v_{P,L}$. In particular, we have $p\leq p_{max}$ From now on, the final function to optimize is $F_{\mb A, \mb v,P,L}$ and the oracle is a standard first-order oracle for this function, using the subgradient function $\partial F_{\mb A, \mb v,P,L}$.

We will relate this procedure to the standard convex optimization problem and prove query lower bounds under memory constraints for this procedure. Before doing so, we formally define what we mean by solving this optimization procedure.

\begin{definition}
Let $alg$ be an algorithm for convex optimization. We say that an algorithm $alg$ is successful for the optimization procedure with probability $q\in[0,1]$ and accuracy $\epsilon>0$, if taking $\mb A\sim\Ucal(\{\pm 1\}^{n\times d})$, running $alg$ with the responses given by the procedure, and denoting by $\mb x^\star(alg)$ the final answer returned by $alg$, with probability at least $q$ over the randomness of $\mb A$ and of the procedure, one has
\begin{equation*}
    F_{\mb A, \mb v, P, L}(\mb x^\star(alg)) \leq \min_{\mb x\in B_d(0,1)} F_{\mb A, \mb v, P, L}(\mb x) +\epsilon.
\end{equation*}
\end{definition}

\subsection{Properties and validity of the optimization procedure}

We begin this section with a simple lemma showing that during each period $p$ at most $l_p\leq k$ vectors $\mb v_{p,1},\ldots,\mb v_{p,l_p}$ are constructed. 

\begin{lemma}\label{lemma:bound_on_l}
    At any time of the construction procedure, $l\leq r$. In particular, since $r\leq k$, we have $l_p\leq k$ for all periods $p\leq p_{max}$.
\end{lemma}

\begin{proof}
    Fix a period $p$. We prove this by induction. The claim is satisfied for any $l=1$ when $p\geq 2$ since in this case, at the first time $t=i_{p,1}$ of the period $p$ we also construct the first vector $\mb v_{p,1}$. For $p=1$, note that the first informative query $t$ that falls in scenarios (2b) or (2c) is exploratory. Indeed, in these cases we have $F_{\mb A, \mb v, 1, 0}(\mb x_t)  < \eta(-\gamma_1-\gamma_2/2) \leq -\eta\gamma_1/2$, and the second criterion for an exploratory query is immediate $\|P_{Span(\mb x_{i_{1,r'}},r'\leq 0)}(\mb x_t)\|=0$ since no indices $i_{1,r'}$ have been defined yet.

    We now suppose that the claim holds for $l-1\geq 1$. Let $t_{p,l}$ be the time when $\mb v_{p,l}$ is constructed and $i_{p,1},\ldots,i_{p,r}$ the indices constructed until the beginning of iteration $t_{p,l}$. If a new index $i_{p,r'}$ was constructed in times $(t_{p,l-1},t_{p,l})$ then the claim holds immediately. Suppose that this is not the case. Note that $t_{p,l}$ falls in scenario (2b) which means in particular that
    \begin{equation*}
        \eta(\mb v_{p,l-1}^\top \mb x_{t_{p,l}} - p\gamma_1-(l-1)\gamma_2)\leq F_{\mb A, \mb v, p,l-1}(\mb x_{t_{p,l}})<\eta(-p\gamma_1-(l-1)\gamma_2-\gamma_2/2).
    \end{equation*}
    As a result,
    \begin{equation*}
        |\mb y_{p,l-1}^\top \mb x_{t_{p,l}}| \geq |\mb v_{p,l-1}^\top \mb x_{t_{p,l}}| - \delta > \frac{\gamma_2}{2} - \delta. 
    \end{equation*}
    Next, when $r\geq l-1$ is the number of indices constructed so far, we decompose $\mb y_{p,l-1} = \alpha_1\mb b_{p,1}+\ldots + \alpha_r \mb b_{p,r} + \tilde{\mb y}_{p,l-1}$ where $\tilde{\mb y}_{p,l-1}\in Span(\mb x_{i_{p,r'}},r'\leq r)^\perp$. Now by construction of $\mb y_{p,l-1}$ one has $|\alpha_{r'}|\leq d^{-3}$ for all $r'\leq r$. Thus,
    \begin{equation*}
        \| \tilde{\mb y}_{p,l-1} - \mb y_{p,l-1} \| \leq \frac{\sqrt r}{d^3} \leq \frac{1}{d^2\sqrt d}.
    \end{equation*}
    Therefore,
    \begin{equation*}
        \|P_{Span(\mb x_{i_{p,r'}},r'\leq r)^\perp}(\mb x_{t_{p,l}})\| \geq |\tilde{\mb y}_{p,l-1}^\top \mb x_{t_{p,l}}| \geq |\mb y_{p,l-1}^\top \mb x_{t_{p,l}}| - \frac{1}{d^2\sqrt d} > \frac{\gamma_2}{2} - \frac{1}{d^2\sqrt d} - \delta \geq \frac{\gamma_2}{4}.
    \end{equation*}
    As a result, $t_{p,l}$ is exploratory, hence $i_{p,r+1}=t_{p,l}$. This ends the proof of the recursion and the lemma.
\end{proof}

We recall that  $P$ and $L$ denote the last defined period and vector $\mb v_{P,L}$. From Lemma \ref{lemma:bound_on_l}, we have in particular $P\leq p_{max}$ and $L\leq k$. In the next result, we show that with high probability, the returned values and vectors returned by the above procedure are consistent with a first-order oracle for minimizing the function $F_{\mb A, \mb v, P, L}$.

\begin{proposition}\label{prop:true_first_order_procedure}
    Let $\mb A \in \{\pm 1\}^{n\times d}$ and $\mb v_0\in\Dcal_\delta$. On an event $\Ecal$ of probability at least $1-C\sqrt{\log d}/d^2$ on the randomness of the procedure for some universal constant $C>0$, all responses of the optimization procedure are consistent with a first-order oracle for the function $F_{\mb A, \mb v, P, L}$: for any $t\geq 1$, if $(f_t,\mb g_t)$ is the response of the procedure at time $t$ for query $\mb x_t$, then $f_t = F_{\mb A, \mb v, P, L}(\mb x_t)$ and $\mb g_t=\partial F_{\mb A, \mb v, P, L}(\mb x_t)$.
\end{proposition}

\begin{proof}
    Consider a given iteration $t$. We aim to show that $(f_t,\mb g_t) = (F_{\mb A, \mb v, P, L}(\mb x_t),\partial F_{\mb A, \mb v, P, L}(\mb x_t))$. By construction, if $t\geq d^2$, the result is immediate. Now suppose $t\leq d^2$. We first consider the case when $\mb x_t$ is non-informative (1). By definition, $F_{\mb A}(\mb x_t) > \eta$. Since for any $(p,l)\leq_{lex}(P,L)$ one has $|\mb v_{p,l}^\top \mb x_t| \leq \|\mb v_{p,l}\|\|\mb x_t\| \leq 1$, we have
    \begin{equation*}
        F_{\mb A, \mb v, P, L}(\mb x_t) = \max\left\{ F_{\mb A}(\mb x_t), \eta\left(\max_{(p,l)\leq_{lex} (P,L)} \mb v_{p,l}^\top \mb x - p\gamma_1 - l\gamma_2\right) \right\} = F_{\mb A}(\mb x_t).
    \end{equation*}
     As a result, the response of the procedure for $\mb x_t$ is consistent with $F_{\mb A, \mb v, P, L}$ and the returned subgradient is $\tilde{\mb g}_{\mb A}(\mb x_t) = \partial F_{\mb A, \mb v, P, L}(\mb x_t)$. Therefore, it suffices to focus on informative queries (2). We will denote by $t_{p,l}$ the index of the iteration when $\mb v_{p,l}$ has been defined, for $(p,l)\leq_{lex}(P,L)$. Consider a specific couple $(p,l)\leq_{lex}(P,L)$, and let $r$ denote the number of constructed indices on or before $t_{p,l}$. Let $\mb b_{p,1},\ldots,\mb b_{p,r}$ the corresponding vectors resulting from the Gram-Schmidt procedure on $\mb x_{i_{p,1}},\ldots, \mb x_{i_{p,r}}$. Then, conditionally on the history until time $t_{p,l}$, the vector $\mb v_{p,l}$ was defined as $\mb v_{p,l} = \phi_\delta(\mb y_{p,l})$, where $\mb y_{p,l} $ is sampled as $\sim\Ucal(S^{d-1}\cap \{ \mb z\in \Rbb^d : |\mb b_{p,r'}^\top \mb z | \leq d^{-3},\forall r'\leq r \})$. As a result, from Lemma \ref{lemma:concentration_bound}, for any $t\leq t_{p,l}$, we have
    \begin{equation*}
        \Pbb\left(|\mb x_t^\top \mb v_{p,l}| \geq 3\sqrt{\frac{2\log d}{d}}  + \frac{2}{d^2} \right) \leq \frac{6\sqrt{2\log d}}{d^6}.
    \end{equation*}
    We then define the following event
    \begin{equation*}
        \Ecal = \bigcap_{(p,l)\leq_{lex}(P,L)}\bigcap_{t\leq t_{p,l}}\left\{|\mb x_t^\top \mb v_{p,l}| < 3\sqrt{\frac{2\log d}{d}} +\frac{2}{d^2} \right\},
    \end{equation*}
    which by the union bound has probability $\Pbb(\Ecal) \geq 1-3\sqrt{2\log d}/d^2$. We are now ready to show that the construction procedure is consistent with optimizing $F_{\mb A, \mb v, P, L}$ on the event $\Ecal$. As seen before, we can suppose that $\mb x_t$ is informative (2). Using the same notations as before, because $\Ecal$ is met, for any $p<p'\leq P$ and $l'\leq l_{p'}$, we have for $d\geq 2$,
    \begin{align*}
        \mb v_{p',l'}^\top \mb x_t-p'\gamma_1 - l'\gamma_2 < 3\sqrt{\frac{2\log d}{d}} + \frac{1}{d} - p\gamma_1 - \gamma_1 \leq -p\gamma_1 - \frac{\gamma_1}{2} \leq -p\gamma_1 - d\gamma_2 - \frac{\gamma_2}{2},
    \end{align*}
    where we used $3\sqrt 2 + 1\leq 6$ and $2d\gamma_2 \leq \gamma_1/2$. As a result, we obtain that
    \begin{equation*}
        \max_{(p',l')\leq_{lex}(P,L), p'>p} \mb v_{p',l'}^\top \mb x_t - p'\gamma_1 -l'\gamma_2 < -p\gamma_1 -l\gamma_2- \frac{\gamma_2}{2}.
    \end{equation*}
    Next, we consider the case of vectors $\mb v_{p,l'}$ where $l\leq l'\leq l_p$ and $t_{p,l'} \geq t$ (this also includes the case when we defined $\mb v_{p,l}$ at time $t=t_{p,l}$). We write $\tilde l$ for the smallest such index $l$. As a remark, $\tilde l\in\{l,l+1\}$. Note that if such indices exist, this means that before starting iteration $t$, the procedure has not yet reached $r=k$. There are two cases. If $\mb x_t$ was exploratory, we have $t=i_{p,r}$ hence $ \|P_{Span(\mb b_{p,r'},r'\leq r)^\top}(\mb x_t)\| = 0$. If $\mb x_t$ is not exploratory, either
    \begin{equation}\label{eq:exploratory_in_space}
        \|P_{Span(\mb b_{p,r'},r'\leq r)^\top}(\mb x_t)\| < \frac{\gamma_2}{4} \|\mb x_t\| \leq \frac{\gamma_2}{4},
    \end{equation}
    or we have $F_{\mb A, v, p, l}(\mb x_t) > -\eta\gamma_1/2$. We start with the last scenario when $F_{\mb A, v, p, l}(\mb x_t) > -\eta\gamma_1/2$. Then, on $\Ecal$, one has
    \begin{equation*}
        \max_{(p,l)<_{lex}(p',l')\leq_{lex}(P,L)} \mb v_{p',l'}^\top \mb x_t - p'\gamma_1 - l'\gamma_2 \leq -\gamma_1 + 3\sqrt{\frac{2\log d}{d}} + \frac{1}{d} \leq -\frac{\gamma_1}{2}
    \end{equation*}
    As a result, this shows that $F_{\mb A, \mb v, P, L}(\mb x_t) = F_{\mb A, \mb v, p, l}(\mb x_t)$. Hence using a first-order oracle from $F_{\mb A, \mb v, l, p}$ at $\mb x_t$ is already consistent with $F_{\mb A, \mb v, P, L}$. Thus, for whichever step (2a), (2b) or (2c) is performed, since these can only increase the knowledge on $\mb v$, the response given by the construction procedure is consistent with minimizing $F_{\mb A, \mb v}$.

    It remains to treat the first two scenarios in which we always have Eq~\eqref{eq:exploratory_in_space}. In particular, when writing $\mb x_t = \alpha_1 \mb b_{p,1} +\ldots + \alpha_r \mb b_{p,r} + \tilde{\mb x}_t$ where $\tilde{\mb x}_t = P_{Span(\mb b_{p,r'},r'\leq r)^\perp}(\mb x_t)$, we have $\|\tilde{\mb x}_t\| < \frac{\gamma_2}{4}$. As a result, for $\tilde l \leq l'\leq l_p$, one has for
    \begin{align*}
    |\mb v_{p,l'}^\top \mb x_t| \leq |\mb y_{p,l'}^\top \mb x_t| + \delta &\leq |\alpha_1||\mb y_{p,l'}^\top\mb b_{p,1}| + \ldots + |\alpha_r||\mb y_{p,l'}^\top \mb b_{p,r}| + \|\tilde{\mb x}_t\| + \delta\\
        &< \|\mb \alpha\|_1 \frac{1}{d^3} + \frac{\gamma_2}{4} +\delta\\
        &\leq \frac{\gamma_2}{4} + \frac{1}{d^2\sqrt d} + \frac{1}{d^3} \leq \frac{\gamma_2}{2},
\end{align*}
where in the last inequality we used $d\geq 3$. As a result, provided that $\tilde l$ exists, this shows that
    \begin{equation}\label{eq:estimate1}
        \max_{\tilde l \leq l'\leq l_p} \mb v_{p,l'}^\top \mb x_t - p\gamma_1 -l'\gamma_2 = \mb v_{p,\tilde l}^\top \mb x_t - p\gamma_1 -\tilde l \gamma_2 <  -p\gamma_1 - \tilde l\gamma_2 + \frac{\gamma_2}{2}.
    \end{equation}
    On the other hand, if $t=i_{p+1,1}$, the same reasoning works for $t$ viewing it as in period $p+1$, which shows for this case that
    \begin{equation}\label{eq:estimate2}
        \max_{l'\leq l_{p+1}} \mb v_{p+1,l'}^\top \mb x_t - (p+1)\gamma_1 -l'\gamma_2 = \mb v_{p+1,1}^\top \mb x_t - (p+1)\gamma_1 -\gamma_2 <  -(p+1)\gamma_1 - \frac{\gamma_2}{2}.
    \end{equation}
    As a conclusion of these estimates, we showed that on $\Ecal$, we have
    \begin{equation*}
        F_{\mb A, \mb v, P, L}(\mb x_t) = \max\left\{  F_{\mb A, \mb v, p, l}(\mb x_t), \eta(\mb v_{p',l'}^\top \mb x_t - p'\gamma_1-l'\gamma_2 )\right\}:=\tilde F_{\mb A, \mb v, t}(\mb x_t)
    \end{equation*}
    where $(p',l')$ is the very next vector that is defined after starting iteration $t$ (potentially, it has $t_{p',l'}=t$ if we defined a vector at this time). It now suffices to check that the value and vector returned by the procedure are consistent with the right-hand side. By construction, if we constructed $\mb v_{p',l'}$ at step $t$: case (2b) or (2c), then the procedure directly uses a first-order oracle for $\tilde F_{\mb A, \mb v, t}$. Further, by construction of the subgradients since they break ties lexicographically in $(p,l)$, the returned subgradient is exactly $\partial F_{\mb A, \mb v, P,L}(\mb x_t)$. It remains to check that this is the case when no vector $\mb v_{p',l'}$ is defined at step $t$: case (2a). This corresponds to the case when $F_{\mb A, \mb v, p, l}(\mb x_t) \geq \eta(-p\gamma_1 - l\gamma_2 - \gamma/2)$. Now in this case, the upper bound estimates from Eq~\eqref{eq:estimate1} and Eq~\eqref{eq:estimate2} imply that
    \begin{equation*}
        \mb v_{p',l'}^\top \mb x_t - p'\gamma_1-l'\gamma_2  < -p\gamma_1 - l\gamma_2 - \gamma/2,
    \end{equation*}
    and as a result, $F_{\mb A, \mb v, P, L}(\mb x_t) = F_{\mb A, \mb v, p, l}(\mb x_t)$. Therefore, using a first-order oracle of $F_{\mb A, \mb v, p, l}$ at $\mb x_t$ is valid, and the break of ties of the subgradient of $\tilde F_{\mb A, \mb v, t}$ is the same as the break of ties of $\partial F_{\mb A, \mb v, P, L}(\mb x_t)$. This ends the proof that on $\Ecal$ the procedure gives responses consistent with an optimization oracle for $F_{\mb A, \mb v, P, L}$ with subgradient function $\partial F_{\mb A, \mb v, P, L}$. Because $\Pbb(\Ecal)\geq 1-C\sqrt{\log d}/d^2$ for some constant $C>0$, this ends the proof of the proposition.
\end{proof}

Last, we provide an upper bound on the optimal value of $F_{\mb A, \mb v, P,L}$.

\begin{proposition}\label{prop:low_optimum}
    Let $\mb A\sim\Ucal(\{\pm 1\}^{n\times d})$ and $\mb v_0\sim \Ucal(\Dcal_\delta)$. For any algorithm $alg$ for convex optimization, let $\mb v$ be the resulting set of vectors constructed by the randomized procedure. With probability at least $1-C\sqrt{\log d}/d$ over the randomness of $\mb A$, $\mb v_0$ and $\mb v$, we have
    \begin{equation*}
        \min_{\mb x\in B_d(0,1)} F_{\mb A, \mb v}(\mb x) \leq -\frac{\eta}{40\sqrt{(kp_{max}+1)\log d}},
    \end{equation*}
    for some universal constant $C>0$.
\end{proposition}

\begin{proof}
    For simplicity, let us enumerate all the constructed vectors $\mb v_1,\ldots,\mb v_{l_{max}}$ by order of construction. Hence, $l_{max} \leq  p_{max}k$. We use the same numerotation for $\mb y_1,\ldots,\mb y_{l_{max}}$. Now let $C_d = \sqrt{40(l_{max}+1)\log d}$ and consider the following vector.
    \begin{equation*}
        \bar{\mb x} = -\frac{1}{C_d}\sum_{l=0}^{l_{max}} P_{Span(\mb a_i, i\leq n)^\perp}(\mb v_l).
    \end{equation*}
    In particular, note that we included $\mb v_0$ in the sum.
    For convenience, we write $P_{\mb A^\perp}$ instead of $P_{Span(\mb a_i,i\leq n)^\perp}$. Also, for convenience let us define $\mb z_l = \sum_{l'\leq l}P_{\mb A^\perp}(\mb v_l)$. Fix an index $1\leq l\leq l_{max}$. Then, by Lemma \ref{lemma:concentration_bound}, with $t_0 := \sqrt{\frac{6\log d}{d}}+\frac{2}{d^2}$, we have
    \begin{align*}
        \Pbb\left(|P_{\mb A^\perp}(\mb v_{l+1})^\top \mb z_l| > t_0 \|\mb z_l\| \right) &=\Pbb\left(|\mb v_{l+1}^\top P_{\mb A^\perp}(\mb z_l)| > t_0 \|\mb z_l\| \right)\\
        &\leq \Pbb\left(|\mb v_{l+1}^\top P_{\mb A^\perp}(\mb z_l)| > t_0 \|P_{\mb A^\perp}(\mb z_l)\| \right)\\
        &\leq \frac{2\sqrt{6\log d}}{d^2}.
    \end{align*}
    Similarly, we have that
    \begin{equation*}
        \Pbb\left(|\mb v_{l+1}^\top \mb z_l| > t_0 \|\mb z_l\| \right) \leq \frac{2\sqrt{6\log d}}{d^2}.
    \end{equation*}
    Now consider the event $\Ecal = \bigcap_{l\leq l_{max}}\{|\mb v_l^\top\mb z_{l-1}|,|P_{\mb A^\perp}(\mb v_l)^\top\mb z_{l-1}|\leq t_0 \|\mb z_l\|\}$, which since $l_{max} \leq d$, by the union bound has probability at least $1-4\sqrt{6\log d}/d$. Then, on $\Ecal$, for any $l<l_{max}$,    
    \begin{equation*}
        \|\mb z_{l+1}\|^2 \leq \|\mb z_l\|^2 + \|P_{\mb A^\perp}(\mb v_{l+1})\|^2 + 2 |P_{\mb A^\perp}(\mb v_{l+1})^\top \mb z_l| \leq \|\mb z_l\|^2 +1+ 2 t_0\|\mb z_l\|.
    \end{equation*}
    We now prove by induction that $\|\mb z_l\|^2\leq 40 \log d\cdot  (l+1)$, which is clearly true for $\mb z_0$ since $\|\mb z_0\| = \|P_{\mb A^\perp}(\mb v_0)\| \leq \|\mb v_0\|\leq 1$. Suppose this is true for $l<l_{max}$. Then, using the above equation and the fact that $t_0\leq 3\sqrt{\frac{\log d}{d}}$ for $d\geq 4$,
    \begin{equation*}
        \|\mb z_{l+1}\|^2 \leq  40\log d\cdot (l+1) + 1 + 6\sqrt{40}\log d\sqrt{\frac{l+1}{d}}\leq 40\log d\cdot (l + 2),
    \end{equation*}
    where we used $l_{max}+1\leq d$, which completes the induction. In particular, on $\Ecal$, we have that $\|\bar{\mb x}\|\leq 1$. Now observe that by construction $\bar{\mb x}\in Span(\mb a_i,i\leq n)^\perp$ so that $\|\mb A\bar{\mb x}\|_\infty =0$. Next, for any $0\leq l\leq l_{max}$, we have
    \begin{equation*}
        \mb v_l^\top \bar{\mb x} = -\frac{\mb v_l^\top \mb z_{l_{max}}}{C_d} = -\frac{1}{C_d}\left(\|P_{\mb A^\perp}(\mb v_l)\|^2 + \mb v_l^\top \mb z_{l-1} + \sum_{l<l'\leq l_{max}} \mb v_l^\top P_{\mb A^\perp}(\mb v_{l'}) \right).
    \end{equation*}
    We will give estimates on each term of the above equation. First, if the indices $i_{p,1},\ldots,i_{p,r}$ were defined before defining $\mb v_l$, we denote $\tilde{\mb y} = P_{Span(\mb x_{i_{p,r'}},r'\leq r)^\perp}(\mb y_l)$, the component of $\mb y_l$ which is perpendicular to the explored space at that time. Then, we can write $\mb y_l = \alpha_1^l \mb b_{p,1}+\ldots + \alpha_r^l \mb b_{p,1} + \tilde{\mb y}_l$, and note that
    \begin{equation*}
        \|\tilde{\mb y}_l\| = \sqrt{\|\mb y_l\|-(\alpha_1^l)^2 -\ldots - (\alpha_r^l)^2} \geq \sqrt{1-\frac{k}{d^6}} \geq 1-\frac{1}{d^5}.
    \end{equation*}
    Then, we have
    \begin{align*}
        \|P_{\mb A^\perp}(\mb v_l)\| &\geq \|P_{\mb A^\perp}(\mb y_l)\| - \delta\\
        &\geq \|P_{Span(\mb a_i, i\leq n,\, \mb b_{p,r'}, r\leq r')^\perp}(\mb y_l)\| - \delta\\
        &= \|P_{Span(\mb a_i, i\leq n,\, \mb b_{p,r'}, r\leq r')^\perp}(\tilde{\mb y}_l)\| - \delta\\
        &\geq \left\|P_{Span(a_i,i\leq n,\, \mb b_{p,r'},r'\leq r)^\perp}\left(\frac{\tilde {\mb y}_l}{\|\tilde{\mb y}_l\|}\right)\right\| - \frac{1}{d^5} -\delta.
    \end{align*}
    As a result, since $\delta =d^{-3}$, this shows that
    \begin{equation*}
        \|P_{\mb A^\perp}(\mb v_l)\|^2 \geq \left\|P_{Span(a_i,i\leq n,\, \mb b_{p,r'},r'\leq r)^\perp}\left(\frac{\tilde {\mb y}_l}{\|\tilde{\mb y}_l\|}\right)\right\|^2 - 2\delta.
    \end{equation*}
    Now observe that $dim(Span(a_i,i\leq n,\, \mb b_{p,r'},r'\leq r)^\perp)\geq d-n-k$, while $\frac{\tilde{\mb y}_l}{\|\tilde{\mb y}_l\|}$ is a uniformly random unit vector in $Span(\mb b_{p,r'}, r\leq r')^\perp$. 
    Therefore, using Proposition \ref{prop:projection_concentration} we obtain for $t<1$,
    \begin{align*}
        \Pbb&\left(\|P_{\mb A^\perp}(\mb v_l)\|^2 +2\delta -\frac{d-n-k}{d} \leq -t\right) \\
        &\leq \Pbb\left(\left\|P_{Span(a_i,i\leq n,\, \mb b_{p,r'},r'\leq r)^\perp}\left(\frac{\tilde {\mb y}_l}{\|\tilde{\mb y}_l\|}\right)\right\|^2-\frac{d-n-k}{d}\leq -t\right)\\
        &\leq e^{-(d-k)t^2}.
    \end{align*}
    As a result since $d-n-k\geq d/2$, we obtain
    \begin{equation*}
        \Pbb\left(\|P_{\mb A^\perp}(\mb v_l)\|^2\leq \frac{1}{2} -2\sqrt{\frac{\log d}{d}} -2\delta\right) \leq \frac{1}{d^2}.
    \end{equation*}
    Now define $\Fcal = \bigcap_{l\leq l_{max}} \{ \|P_{\mb A^\perp}(\mb v_l)\|^2 \geq \frac{1}{2} - 2\sqrt{\frac{\log d}{d}} -2\delta \}$, which since $l_{max}+1\leq d$ and by the union bound has probability at least $\Pbb(\Fcal)\geq 1-1/d$. Next, we turn to the last term. For any $0\leq l<l_{max}$, we now focus on the sequence $(\sum_{l'=l+1}^{l+u}\mb v_l^\top P_{\mb A^\top}(\mb y_{l'}))_{1\leq u\leq l_{max}-l}$ and first note that this is a martingale. These increments are symmetric (because $\mb y_{l'}$ is symmetric) even conditionally on $\mb A$ and $\mb v_l,\mb y_l,\ldots, \mb y_{l'-1}$. Next, let $t_1 = 2\sqrt{\frac{3\log d}{d}} +\frac{2}{d^2}$. Note that for $d\geq 4$, we have $t_1\leq 4\sqrt{\frac{\log d}{d}}$. Further, by Lemma \ref{lemma:concentration_bound},
    \begin{equation*}
        \Pbb(|\mb v_l^\top P_{\mb A^\top}(\mb y_{l'})|>t_1) =\Pbb(|P_{\mb A^\top}(\mb v_l)^\top \mb y_{l'}|>t_1 ) \leq \frac{4\sqrt{3\log d}}{d^4},
    \end{equation*}
    where we used the fact that $P_{\mb A^\perp}$ is a projection. Let $\Gcal_l = \bigcap_{l<l'\leq l_{max}}\{|\mb v_l^\top P_{\mb A^\top}(\mb v_{l'})|\leq  t_1\}$, which by the union bound has probability $\Pbb(\Gcal_l)\geq 1-4\sqrt{3\log d}/d^3$. Next, we define $I_{l,u} = (\mb v_l^\top P_{\mb A^\top}(\mb y_{l+u})\wedge  t_1)\vee(-t_1)$, the increments capped at absolute value $t_1$. Because $\mb v_l^\top P_{\mb A^\top}(\mb y_{l+u})$ is symmetric, so is $I_{l,u}$. As a result, these are bounded increments of a martingale, to which we can apply the Azuma-Hoeffding inequality.
    \begin{equation*}
        \Pbb\left(\left|\sum_{u=1}^{l_{max}-l} I_{l,u} \right| \leq 2t_1\sqrt{(l_{max}-l)\log d} \right) \geq 1-\frac{2}{d^2}.
    \end{equation*}
    We denote by $\Hcal_l$ this event. Now observe that on $\Gcal_l$, the increments $I_{l,u}$ and $\mb v_l^\top P_{\mb A^\top}(\mb y_{l+u})$ coincide for all $1\leq u\leq l_{max}-l$. As a result, on $\Gcal_l\cap\Hcal_l$ we obtain
    \begin{align*}
        \left|\sum_{l<l'\leq l_{max}} \mb v_l^\top P_{\mb A^\perp}(\mb v_{l'})\right| &\leq 
        \left|\sum_{l<l'\leq l_{max}} \mb v_l^\top P_{\mb A^\perp}(\mb y_{l'})\right| + (l_{max}-1)\delta\\
        &\leq  \left|\sum_{u=1}^{l_{max}-l} I_{l,u} \right| + (d-2)\delta\\
        &\leq 2t_1 \sqrt{l_{max}\log d} + (d-2)\delta.
    \end{align*}
    Then, on the event $\Ecal\cap\Fcal\cap\bigcap_{l\leq l_{max}}\Gcal_l\cap\Hcal_l$, for any $1\leq l\leq l_{max}$ one has
    \begin{align*}
        \mb v_l^\top \mb z_{l_{max}} &\geq \frac{1}{2}- 2\sqrt{\frac{\log d}{d}}-t_0\|\mb z_l\| -2t_1 \sqrt{l_{max}\log d}-\frac{1}{d^2}  \\
        &\geq  \frac{1}{2}- 2\sqrt{\frac{\log d}{d}}-3\log d \sqrt{40\frac{l_{max}+1}{d}} -8\log d \sqrt{\frac{l_{max}}{d}} - \frac{1}{d^2} \\
        &\geq \frac{1}{2} - 30\log d\sqrt{\frac{l_{max}+1}{d}}\\
        &\geq \frac{1}{6},
    \end{align*}
    where in the last inequalities we used the fact that $l_{max} \leq kp_{max}\leq c_{d,1} d-1$ where $c_{d,1} = \frac{1}{90^2\log^2 d}$ as per Eq~\eqref{eq:definition_p_max}. As a result, we obtain that on $\Ecal\cap\Fcal\cap\bigcap_{l\leq l_{max}}\Gcal_l\cap\Hcal_l$, which has probability at most $1-C\sqrt {\log d}/d$ for some constant $C>0$,
    \begin{equation*}
        \max_{p\leq p_{max}, l\leq k} \mb v_{p,l}^\top \bar{\mb x} \leq -\frac{1}{6C_d} \leq -\frac{1}{40\sqrt{(kp_{max}+1)\log d}}.
    \end{equation*}
    Since $\|\mb A\bar{\mb x}\|_\infty=0$, and $\eta\geq  \frac{\eta}{40\sqrt{(kp_{max}+1)\log d}}$, this shows that
    \begin{equation*}
        F_{\mb A, \mb v}(\bar{\mb x}) \leq -\frac{\eta}{40\sqrt{(kp_{max}+1)\log d}}.
    \end{equation*}
    This ends the proof of the proposition.
\end{proof}

\subsection{Reduction from convex optimization to the optimization procedure} 

According to Proposition \ref{prop:true_first_order_procedure}, with probability at least $1-C\sqrt{\log d}/d^2$, the procedure returns responses that are consistent with a first-order oracle of the function $F_{\mb A, \mb v, P , L}$ where $\mb v_{P,L}$ is the last vector to have been defined. Now observe that for any constructed vectors $\mb v$, the function $F_{\mb A, \mb v, P, L}$ is $\sqrt d$-Lipschitz. As a result, if there exists an algorithm for convex optimization that guarantees $\epsilon$ accuracy for $1$-Lipschitz functions, by rescaling, there exists an algorithm $alg$ which is successful for the optimization procedure with probability $1-C\sqrt{\log d}/d^2$ and $\epsilon\sqrt d$ accuracy. In the next proposition, we show that to be successful, such an algorithm needs to properly define the complete function $F_{\mb A, \mb v}$, i.e., to complete all periods until $p_{max}$.

\begin{proposition}\label{prop:necessary_vectors}
    Let $alg$ be a successful algorithm for the optimization procedure with probability $q\in[0,1]$ and precision $\eta/(2\sqrt d)$. Suppose that $alg$ performs at most $d^2$ queries during the optimization procedure. Then when running $alg$ with the responses of the optimization procedure, $alg$ succeeds and ends the period $p_{max}$ with probability at least $q-C\sqrt{\log d}/d$ for some universal constant $C>0$.
\end{proposition}

\begin{proof}
    Let $\mb x^\star(alg) = \mb x_T$ denote the final answer of $alg$ when run with the optimization procedure. By hypothesis, we have $T\leq d^2$. As before, let $P\leq p_{max}$ and $L\leq k$ be the indices such that the last vector constructed by the optimization procedure is $v_{P,L}$. Let $\Ecal$ be the event when $alg$ run on the optimization procedure does not end period $p_{max}$. We focus on $\Ecal$ and consider two cases. 
    
    First, suppose that $T>t_{P,L}$, i.e., the last vector was not constructed at time $T$. As a result, this means that $\mb x_T$ corresponds either to a non-informative query---scenario (1)---in which case $F_{\mb A, \mb v, P, L}(\mb x_T)\geq F_{\mb A}(\mb x_T) \geq \eta$, or this means that $F_{\mb A, \mb v, P, L}(\mb x_t) \geq \eta(-P\gamma_1 - L\gamma_2-\gamma/2)$---scenario (2a).
    
    Second, we now suppose that $T=t_{P,L}$, i.e., the last vector was constructed at time $T$. Then, by construction of $\mb v_{P,L}$ and $\mb y_{P,L}$, we have indices $i_{P,1},\ldots,i_{P,r}\leq T$ such that with the Gram-Schmidt decomposition $\mb b_{P,1},\ldots,\mb b_{P,r}$ of $\mb x_{i_{P,1}},\ldots, \mb x_{i_{P,r}}$, we have $|\mb b_{p,r'}^\top \mb y_{P,L}|\leq d^{-3}$ for all $r'\leq r$. In particular, writing $\mb x_T = \alpha_1 \mb b_{P,1} + \ldots + \alpha_r \mb b_{P,r}+\tilde{\mb x}_T$, where $\tilde{\mb x}_T\in Span(\mb x_{i_{P,r'}},r'\leq r)^\perp$, either we have $i_{P,r}=T$, in which case $\tilde{\mb x}_T=\mb 0$, or $\mb x_T$ was not exploratory in which case we directly have $F_{\mb A, \mb v, P, L}(\mb x_T) \geq F_{\mb A, \mb v, P, L-1}(\mb x_T) >-\eta\gamma_1/2$, or we have $\|\tilde{\mb x}_T\|<\|\mb x_T\|\gamma_2/4 \leq \gamma_2/4$. For all remaining cases to consider, we obtain
    \begin{equation*}
         |\mb v_{P,L}^\top \mb x_T|\leq |\mb y_{P,L}^\top \mb x_T|+\delta \leq \frac{\|\mb \alpha\|_1}{d^3} + \|\tilde{\mb x}_T\| +\delta\leq \frac{1}{d^3}+\frac{1}{d^2\sqrt d} + \frac{\gamma_2}{4} <\frac{\gamma_2}{2}.
    \end{equation*}
    In the last inequality, we used $d\geq 4$. This shows that $F_{\mb A, \mb v, P, L}(\mb x_T) \geq \eta(-P\gamma_1 - L\gamma_2-\gamma_2/2)$. As a result, in all cases this shows that $F_{\mb A, \mb v, P, L}(\mb x^\star(alg)) \geq \eta(-P\gamma_1 - L\gamma_2-\gamma_2/2)\geq -\eta(p_{max}+1)\gamma_1$. Now define the event
    \begin{equation*}
        \Fcal = \left\{ \min_{\mb x\in B_d(0,1)} F_{\mb A, \mb v}(\mb x)\leq -\frac{\eta}{40\sqrt{(kp_{max}+1)\log d}} \right\}.
    \end{equation*}
    By Proposition \ref{prop:low_optimum} we have $\Pcal(\Fcal) \geq 1-C\sqrt{\log d}/d$. Now from Eq~\eqref{eq:definition_p_max},
    \begin{equation*}
        (p_{max}+1)^{3/2} \leq \frac{1}{60 \gamma_1\sqrt{k\log d}}.
    \end{equation*}
    Thus,
    \begin{equation*}
        (p_{max}+1)\gamma_1 \leq \frac{1}{60\sqrt{k(p_{max}+1)\log d}} \leq \frac{1}{60\sqrt{(kp_{max}+1)\log d}}
    \end{equation*}
    Then, since $F_{\mb A, \mb v, P, L} \leq F_{\mb A, \mb v}$, this shows that on $\Ecal\cap\Fcal$,
    \begin{align*}
        F_{\mb A, \mb v, P, L}(\mb x^\star(alg)) \geq -\eta(p_{max}+1)\gamma_1 &\geq \min_{\mb x\in B_d(0,1)} F_{\mb A, \mb v}(\mb x) + \frac{\eta}{120\sqrt{(kp_{max}+1)\log d}} \\
        &> \min_{\mb x\in B_d(0,1)} F_{\mb A, \mb v, P, L}(\mb x) + \frac{\eta}{2\sqrt d}
    \end{align*}
    where in the last inequality, we used $kp_{max}\leq c_{d,1}d-1$.
    As a result, letting $\Gcal$ be the event when $alg$ succeeds for precision $\epsilon = \eta/(2\sqrt d)$. By hypothesis, $\Pcal(\Gcal)\geq q$. Now from the above equations, one has $\Ecal\cap\Fcal\cap\Gcal=\emptyset$. Therefore, $\Pbb(\Gcal\cap\Ecal^c) \geq \Pcal(\Gcal) - \Pbb(\Gcal\cap\Ecal\cap \Fcal) - \Pbb(\Fcal^c) \geq q-C\sqrt{\log d}/d.$ This ends the proof of the proposition.
\end{proof}

\subsection{Reduction of the optimization procedure to an Orthogonal Vector Game with Hints}

We are now ready to introduce an orthogonal vector game where the main difference with the game introduced in \cite{marsden2022efficient} is that the player can provide additional hints.

\begin{game}[ht]

\caption{Orthogonal Vector Game with Hints}\label{game:hint_game}

\setcounter{AlgoLine}{0}

\SetAlgoLined
\LinesNumbered

\everypar={\nl}

\hrule height\algoheightrule\kern3pt\relax
\KwIn{$d$, $k$, $m$, $M$, $\alpha$, $\beta$}

\textit{Oracle:} Set $n\gets \lfloor d/4 \rfloor$, sample $\mb A \sim \Ucal( \{\pm 1\}^{n\times d} )$.\;

\textit{Player:} Observe $\mb A$\;

\For{$l\in [d]$}{
    \textit{Player:} Based on $\mb A$ and any previous queries and responses, submit at most $k$ vectors $\mb x_{l,1},\ldots,\mb x_{l,r_l}$.\;

    \textit{Oracle:} Perform the Gram-Schmidt decomposition $\mb b_{l,1},\ldots,\mb b_{l,r_l}$ of $\mb x_{l,1},\ldots,\mb x_{l,r_l}$. Then, sample a vector $\mb y_l\in S^{d-1}$ according to a uniform distribution $\Ucal(S^{d-1}\cap \{\mb z\in \Rbb^d : \forall r\leq r_l,|\mb b_{l,r}^\top \mb z|\leq d^{-3}\})$. As response to the query, return $\mb v_l = \phi_\delta(\mb y_l)$ to the player.

}

\textit{Player:} Based on $\mb A$, all previous queries and responses, store an $M$-bit message $\mathsf{Message}$.\;

\textit{Player:} Based on $\mb A$, all previous queries and responses, submit a function $\mb g:B_d(0,1)\to (\{\mb a_j,j\leq n\}\cup \{\mb v_l,l\leq d\})\times [d^2]$ to the Oracle.

\For{$i\in[m]$}{
  \textit{Player:} Based on $\mathsf{Message}$, any previous queries $\mb x_1,\ldots, \mb x_{i-1}$ and responses $\mb g_1,\ldots, \mb g_{i-1}$ from this loop phase, submit a query $\mb x_i\in\Rbb^d$.\;
  
  \textit{Oracle:} As the response to query $\mb z_i$, return $\mb g_i = \mb g(\mb z_i)$.
}
\textit{Player:} Based on all queries and responses from this phase $\{\mb z_i, \mb g_i, i\in[m]\}$, and on $\mathsf{Message}$, return some vectors $\mb y_1,\ldots, \mb y_k$ to the oracle.\;

The player wins if the returned vectors have unit norm and satisfy for all $i\in[k]$
\begin{enumerate}
    \item $\|\mb A \mb y_i\|_\infty \leq \alpha$
    \item $\|P_{Span(\mb y_1,\ldots,\mb y_{i-1})^\perp}(\mb y_i)\|_2 \geq \beta$.
\end{enumerate}
\hrule height\algoheightrule\kern3pt\relax
\end{game}

We first prove that solving the optimization procedure implies solving the Orthogonal Vector Game with Hints.

\begin{proposition}\label{prop:reductionto_orthogonal_vector_game}
    Let $m\leq d$. Suppose that there is an $M$-bit algorithm that is successful for the optimization procedure with probability $q$ for accuracy $\epsilon = \eta/(2\sqrt d)$ and uses at most $m p_{max}$ queries. Then, there is an algorithm for Game \ref{game:hint_game} for parameters $(d, k, m, M,\alpha=\frac{2\eta}{\gamma_1}, \beta = \frac{\gamma_2}{4})$, for which the Player wins with probability at least $q-C\sqrt{\log d}/d$ for some universal constant $C>0$. 
\end{proposition}

\begin{proof}
    Let $alg$ be an $M$-bit algorithm solving the feasibility problem with $mp_{max}$ queries with probability at least $q$. We now describe the strategy for Game \ref{game:hint_game}.

    In the first part of the strategy, the player observes $\mb A$. First, submit an empty query to the Oracle to obtain a vector $\mb v_0$, which as a result is uniformly distributed among $\Dcal_\delta$. We then proceed to simulate the optimization procedure for $alg$ using parameters $\mb A$ and $\mb v_0$ (lines 3-6 of Game \ref{game:hint_game}). Precisely, whenever a new vector $\mb v_{p,l}$ needs to be defined according to the optimization procedure, the player submits the corresponding vectors $\mb x_{i_{p,1}},\ldots,\mb x_{i_{p,r}}$ to the oracle and receives in return a vector which defines $\mb v_{p,l}$. In this manner, the player simulates exactly the optimization procedure. In all cases, the number of queries in this first phase is at most $1+kp_{max}\leq d$. For the remaining queries to perform, the player can query whichever vectors, these will not be used in the rest of the strategy. If the simulation did not end period $p_{max}$, the complete procedure fails. We now describe the rest of the procedure when period $p_{max}$ was ended. During the simulation, the algorithm records the time $i_{p,1}$ when period $p$ started for all $p\leq p_{max}+1$. Recall that for $p_{max}+1$, we only define $i_{p_{max}+1,1}$, this is the time that ends period $p_{max}$. Now by hypothesis, $i_{p_{max}+1,1}\leq mp_{max}$. As a result, there must be a period $p\leq p_{max}$ which uses at most $m$ queries: $i_{p+1,1}-i_{p,1}\leq m$. We define the memory $\mathsf{Message}$ to be the memory of $alg$ just before starting iteration $i_{p,1}$, at the beginning of period $p$ (line 7 of Game \ref{game:hint_game}). Next, since the period $p_{max}$ was ended, the vectors $\mb v_{p,l}$ for $p\leq p_{max},l\leq l_p$ were all defined. The player can therefore submit the function $\mb g_{\mb A, \mb v}$ to the Oracle (line 8 of Game \ref{game:hint_game}) as follows,
    \begin{equation}\label{eq:definition_subgradient_function}
        \mb g_{\mb A, \mb v}:\mb x \mapsto \begin{cases}
            (\mb g_{\mb A}(\mb x),1) &\text{if } F_{\mb A, \mb v}(\mb x) = \|\mb A\mb x\|_\infty - \eta,\\
            (\mb v_0,2) &\text{otherwise and if } F_{\mb A, \mb v}(\mb x) = \eta\mb v_0^\top \mb x,\\
            (\mb v_{p,l},2+(p-1)k + l) &\text{otherwise and if}\\
            & (p,l) = \displaystyle \argmax_{(p',l')\leq_{lex}(p_{max},l_{p_{max}})} \mb v_{p',l'}^\top \mb x -p\gamma_1-l\gamma_2.
        \end{cases}
    \end{equation}
    Intuitively, the first component of $\mb g_{\mb A, \mb v}$ gives the subgradient $\partial F_{\mb A, \mb v}$ to the following two exceptions: we always return $\mb a_i$ instead of $\pm \mb a_i$ and we return $\mb v_0$ (resp. $\mb v_{p,l}$) instead of $\eta\mb v_0$ (resp. $\eta\mb v_{p,l}$). The second term of $\mb g_{\mb A, \mb v}$ has values in $[2+p_{max}k]$. Hence, since $2+p_{max}k\leq d^2$, the function $\mb g_{\mb A, \mb v}$ takes values in $(\{\mb a_j,j\leq n\}\cup \{\mb v_l,l\leq d\})\times [d^2]$.

    The strategy then proceeds to play the Orthogonal Vector Game in a second part (lines 9-12 of Game \ref{game:hint_game}) and uses the responses of the Oracle to simulate the run of $alg$ for the optimization procedure in period $p$. To do so, we set the memory state of the algorithm $alg$ to be $\mathsf{Message}$. Then, for the next $m$ iterations we proceed as follows. At iteration $i$ of the process, we run $alg$ with its current state to obtain a new query $\mb z_i$ which is then submitted to the oracle of the Orthogonal Vector Game, to get a response $(\mb g_i,s_i)$. We then use this response to simulate the response that was given by the optimization procedure in the first phase, computing $(v_i,\tilde{\mb g}_i)$ as follows
    \begin{equation}\label{eq:definition_response_to_alg}
        (v_i,\tilde{\mb g}_i) = \begin{cases}
            (|\mb g_i^\top \mb z_i|-\eta,sign(\mb g_i^\top \mb z_i)\mb g_i) & s_i=1,\\
            (\eta\mb g_i^\top \mb z_i,\eta\mb g_i) & s_i=2,\\
            (\eta(\mb g_i^\top \mb z_i - p\gamma_1-l\gamma_2),\eta\mb g_i) & s_i=2+(p-1)k+l, p\leq p_{max},1\leq l\leq k.\\
        \end{cases}
    \end{equation}
    We can easily check that in all cases, $v_i = F_{\mb A, \mb v}(\mb z_i)$ and that $\tilde{\mb g}_i = \partial F_{\mb A, \mb v}(\mb z_i)$. We then pass $(\mb v_i, \tilde{\mb g}_i)$ as response to $alg$ for the query $\mb z_i$ so it can update its state. Further, having defined $i_1=1$, the player can keep track of exploratory queries by checking whether
    \begin{equation*}
        v_i\leq -\frac{\eta\gamma_1}{2} \quad \text{ and }\quad \frac{\|P_{Span(\mb z_{i_{r'}},r'\leq r)^\perp}(\mb z_i)\|}{\|\mb z_i\|} \geq \frac{\gamma_2}{4},        
    \end{equation*}
    where $i_1,\ldots,i_r$ are the indices defined so far. We perform $m$ such iterations unless $alg$ stops and use the last remaining queries arbitrarily. Next, we check if the last index $i_k$ was defined. If not, we pose $i_k=m+1$ and let $\mb z_{m+1}$ be the next query of $alg$. The final returned vectors are $\frac{\mb z_{i_1}}{\|\mb z_{i_1}\|},\ldots,\frac{\mb z_{i_k}}{\|\mb z_{i_k}\|}$. This ends the description of the player's strategy.

\begin{algorithm}
        
\caption{Strategy of the Player for the Orthogonal Vector Game with Hints}\label{alg:strategy_optimization}

\setcounter{AlgoLine}{0}
\SetAlgoLined
\LinesNumbered

\everypar={\nl}

\hrule height\algoheightrule\kern3pt\relax
\KwIn{$d$, $k$, $p_{max}$, $m$, algorithm $alg$}

\vspace{5pt}

{\nonl \textbf{Part 1:}} Strategy to store $\mathsf{Message}$ knowing $\mb A$\;

Initialize the memory of $alg$ to be $\mb 0$.\;

Submit $\emptyset$ to the Oracle and use the response as $\mb v_0$.\;

Run $alg$ with the optimization procedure knowing $\mb A$ and $\mb v_0$ until the first exploratory query $\mb x_{i_{1,1}}$.

\For{$p\in[p_{max}]$}{
    Let $\mathsf{Memory}_p$ be the current memory state of $alg$ and $i_{p,1}$ the current iteration step. \;
    
    Run $alg$ with the feasibility procedure until period $p$ ends at iteration step $i_{p+1,1}$. If $alg$ stopped before, \Return the strategy fails. When needed to sample a unit vector $\mb v_{p',l'}$, submit vectors $\mb x_{i_{p',1}},\ldots \mb x_{i_{p',r'}}$ to the Oracle where $i_{p',1},\ldots, i_{p',r'}$ are the exploratory queries defined at that stage. We use the corresponding response of the Oracle as $\mb v_{p',l'}$.\;

    \uIf{$i_{p+1,1}-i_{p,1}\leq m$}{
        Set $\mathsf{Message}=\mathsf{Memory}_p$
    }
}
\lFor{Remaining queries to perform to Oracle}{
    Submit arbitrary query, e.g. $\emptyset$
}

\lIf{$\mathsf{Message}$ has not been defined yet}{\Return The strategy fails}

Submit $\mb g_{\mb A, \mb v}$ to the Oracle as defined in Eq~\eqref{eq:definition_subgradient_function}.\;

\vspace{5pt}

{\nonl \textbf{Part 2:}} Strategy to make queries\;

Set the memory state of $alg$ to be $\mathsf{Message}$ and define $i_1=1$, $r=1$.\;

\For{$i\in[m]$}{
    Run $alg$ with current memory to obtain a query $\mb z_i$.\;
    
    Submit $\mb z_i$ to the Oracle from Game \ref{game:hint_game}, to get response $(\mb g_i,s_i)$.\;
    
    Compute $(v_i,\tilde{\mb g}_i)$ using $\mb z_i$, $\mb g_i$ and $s_i$ as defined in Eq~\eqref{eq:definition_response_to_alg} and pass $(v_i,\tilde{\mb g}_i)$ as response to $alg$.\;

    \uIf{$v_i\leq -\eta\gamma_1/2$ and $\|P_{Span(\mb z_{i_{r'}},r'\leq r)^\perp}(\mb z_i)\| / \|\mb z_i\| \geq \frac{\gamma_2}{4}$}{
        Set $i_{r+1}=i$ and increment $r\leftarrow r+1$.
    }
}

\vspace{5pt}

{\nonl \textbf{Part 3:}} Strategy to return vectors\;

\uIf{index $i_k$ has not been defined yet}{
    With the current memory of $alg$ find a new query $\mb z_{m+1}$ and set $i_k=m+1$.\;
}
\Return $\left\{\frac{\mb z_{i_1}}{\|\mb z_{i_1}\|},\ldots, \frac{\mb z_{i_k}}{\|\mb z_{i_k}\|}\right\}$ to the Oracle.

\hrule height\algoheightrule\kern3pt\relax
\end{algorithm}

    We now show that the player wins with good probability. First, since $alg$ makes at most $mp_{max}\leq d^2$ queries, by Proposition \ref{prop:necessary_vectors}, on an event $\Ecal$ of probability at least $q-C\sqrt{\log d}/d$, $alg$ succeeds and ends the period $p_{max}$. On $\Ecal$, by construction, the first phase of the strategy does not fail. Now we show that in the second phase (lines 9-12 of Game \ref{game:hint_game}), the queried vectors coincide exactly with the queried vectors from the corresponding period $p$ in the first phase (lines 3-6 of Game \ref{game:hint_game}). To do so, we only need to check that the responses provided to $alg$ coincide with the response given by the optimization procedure. First, recall that on $\Ecal$, all periods are completed, hence $F_{\mb A, \mb v, P, L} = F_{\mb A, \mb v}$. Next, by Proposition \ref{prop:true_first_order_procedure}, the responses of the procedure are consistent with optimizing $F_{\mb A, \mb v, P, L}$ and subgradients $\partial F_{\mb A, \mb v, P, L}$ on an event $\Fcal$ of probability at least $1-C'\sqrt{\log d}/d^2$. Therefore, on $\Ecal\cap\Fcal$, it suffices to check that the responses provided to $alg$ are consistent with $F_{\mb A, \mb v}$, which we already noted: at every step $i$, $(v_i,\tilde{\mb g}_i) = (F_{\mb A, \mb v}(\mb z_i),\partial F_{\mb A, \mb v}(\mb z_i))$. This proves that the responses and queries coincide exactly with those given by the optimization procedure on $\Ecal\cap\Fcal$.
    
    Next, by construction, the chosen phase $p$ had at most $m$ iterations. Thus, on $\Ecal\cap\Fcal$, among $\mb z_1,\ldots, \mb z_{m+1}$, we have the vectors $\mb x_{i_{p,1}},\ldots,\mb x_{i_{p,k}}$. Further, if $i_k$ was not defined during part 2 of the strategy, this means that $i_k=m+1$, as defined in the player's strategy (line 21-22 of Algorithm \ref{alg:strategy_optimization}). As a result, for all $u\leq k$, we have $\mb z_{i_u} = \mb x_{i_{p,u}}$. We now show that the returned vectors $\frac{\mb x_{i_{p,1}}}{\|\mb x_{i_{p,1}}\|},\ldots, \frac{\mb x_{i_{p,k}}}{\|\mb x_{i_{p,k}}\|}$ are successful for Game \ref{game:hint_game}. First, because $i_{p,1},\ldots,i_{p,k}$ are exploratory queries, we have directly for $u\leq k$,
    \begin{equation*}
        \frac{\|P_{Span(\mb x_{i_{p,v}},v<u)^\perp}(\mb x_{i_{p,u}})\|}{\|\mb x_{i_{p,u}}\|} \geq \frac{\gamma_2}{4}.
    \end{equation*}
    Next, if $l$ is the index of the last constructed vector $\mb v_{p,l}$ before $i_{p,u}$ in the optimization procedure, one has $F_{\mb A, \mb v, p, l}(\mb x_{i_{p,u}})\leq -\eta\gamma_1/2 $. Therefore, $\|\mb A\mb x_{i_{p,u}}\|_\infty \leq F_{\mb A, \mb v, p, l}(\mb x_{i_{p,u}})+\eta\leq \eta$. Further, $\eta\mb v_0^\top \mb x_{i_{p,u}}\leq F_{\mb A, \mb v, p, l}(\mb x_{i_{p,u}})\leq -\eta\gamma_1/2$. This proves that $\|\mb x_{i_{p,u}}\| \geq \gamma_1/2$. Putting the previous two inequalities together yields
    \begin{equation*}
        \frac{\|\mb A\mb x_{i_{p,u}}\|_\infty}{\|\mb x_{i_{p,u}}\|} \leq \frac{2\eta}{\gamma_1}.
    \end{equation*}
    As a result, this shows that the returned vectors are successful for Game \ref{game:hint_game} for the desired parameters $\alpha = 2\eta/\gamma_1$ and $\beta=\gamma_2/4$. Thus, the player wins on $\Ecal\cap\Fcal$, which has probability at least $q-(C+C')\sqrt{\log d}/d^2$ by the union bound. This ends the proof of the proposition.
\end{proof}

\subsection{Query lower bound for the Orthogonal Vector Game with Hints}

Before proving a lower bound on the necessary number of queries for Game \ref{game:hint_game}, we need to introduce two results. The first one is a known concentration result for vectors in the hypercube. It shows that for a uniform vector in the hypercube, being approximately orthogonal to $k$ orthonormal vectors has exponentially small probability in $k$.

\begin{lemma}[\cite{marsden2022efficient}]
\label{lemma:sensitive_base_marsden}
    Let $\mb h\sim\Ucal(\{\pm 1\}^d)$. Then, for any $t\in(0,1/2]$ and any matrix $\mb Z=[\mb z_1,\ldots,\mb z_k]\in \Rbb^{d\times k}$ with orthonormal columns,
    \begin{equation*}
        \Pbb(\|\mb Z^\top \mb h\|_\infty \leq t)\leq 2^{-c_H k}.
    \end{equation*}
\end{lemma}

We will also need an anti-concentration bound for random vectors, which intuitively provides a lower bound for the previous concentration result. The following lemma shows that for a uniformly random unit vector, being orthogonal to $k$ orthonormal vectors is still achievable with exponentially small probability in $k$.

\begin{lemma}\label{lemma:anti-concentration}
    Let $k< d$ and $\mb x_1,\ldots,\mb x_k$ be $k$ orthonormal vectors. Then,
    \begin{equation*}
        \Pbb_{\mb y\sim\Ucal(S^{d-1})} \left(|\mb x_i^\top \mb y| \leq \frac{1}{d^3}  , \forall i\leq k \right) \geq \frac{1}{e^{d^{-4}}d^{3k}}.
    \end{equation*}
\end{lemma}

\begin{proof}
    Let $\mb y\sim\Ucal(S^{d-1})$ be a uniformly random unit vector. Then, for $i< k$ and any $ y_1,\ldots,  y_{i-1}$ such that $|y_1|,\ldots, |y_{i-1}|\leq \frac{1}{d^3}$, we have
    \begin{align*}
        \Pbb\left(|y_i| \leq \frac{1}{d^3} \mid y_1,\ldots,y_{i-1}\right) & = \Pbb_{\mb u\sim \Ucal(S^{d-i})} \left(|u_1| \leq \frac{1}{d^3\sqrt {1-(y_1^2+\ldots +y_{i-1}^2)}}\right) \\
        &\geq \frac{\int_0^{1/d^3} (1-y^2)^{(d-i-1)/2}dy}{\int_0^1 (1-y^2)^{(d-i-1)/2}dy}\\
        &\geq \frac{(1-d^{-6})^{d/2}}{d^3} \geq \frac{e^{-d^{-5}}}{d^3},
    \end{align*}
    where in the last equation we used $d\geq 2$.
    Therefore, we can show by induction that $\Pbb(|y_i| \leq 1/d^3,\forall i\leq k) \geq \frac{e^{-kd^{-5}}}{d^{3k}}.$ Thus, by isometry this shows that
    \begin{equation*}
        \Pbb\left(|\mb x_i^\top \mb y| \leq \frac{1}{d^3},\forall i\leq k\right) \geq \frac{1}{e^{d^{-4}}d^{3k}}.
    \end{equation*}
    This ends the proof of the lemma.
\end{proof}

We are now ready to prove a query lower bound for Game \ref{game:hint_game}. Precisely, we show that for appropriate choices of parameters, one needs $m=\tilde \Omega(d)$ queries. The proof is closely inspired from the arguments given in \cite{marsden2022efficient}. The main added difficulty arises from bounding the information leakage of the provided hints. As such, our goal is to show that these do not provide more information than the message itself.

\begin{proposition}\label{prop:lower_bound_queries_game}
Let $k\geq 20\frac{M+3d\log(2d)+1}{c_H n}$. And let $0<\alpha,\beta\leq 1$ such that
$\alpha(\sqrt d/\beta)^{5/4}\leq \frac{1}{2}$. If the Player wins the Orthogonal Vector Game with Hints (Game \ref{game:hint_game}) with probability at least $1/2$, then $m\geq \frac{c_H}{8(30\log d+c_H)}d$.
\end{proposition}

\begin{proof}
We first define some notations. Let $\mb Y=[\mb y_1,\ldots, \mb y_k]$ be the matrix storing the final outputs from the algorithm. Next, for the responses of the oracle $(\mb g_1,s_1),\ldots,(\mb g_m,s_m)$, we first store all the scalar responses in a vector $\mb c=[s_1,\ldots,s_m]$. We now focus on the responses $\mb g_1,\ldots,\mb g_m$. Next, let $\tilde{\mb G}$ denote the matrix containing these responses of the oracle which are lines of $\mb A$. Let $\mb G$ be the matrix containing unique columns from $\tilde{\mb G}$, augmented with rows of $\mb A$ so that it has exactly $m$ columns which are all different rows of $\mb A$. Last, let $\mb A'$ be the matrix $\mb A$ once the rows from $\mb G$ are removed. Next, let $\tilde{\mb V}$ be a matrix containing the responses of the oracle which are vectors $\mb v_l$, ordered by increasing index $l$. As before, let $\mb V$ be the matrix $\tilde{\mb V}$ where we only conserve unique columns and append it with additional vectors $\mb v_l$ so that $\mb V$ has exactly $m$ columns. We denote by $\mb w_1,\ldots, \mb w_m$ these vectors, and recall that they are vectors $\mb v_l$ ordered by increasing order of index $l$. Last, we define a vector $\mb j$ of indices such that $j(i)$ contains the information of which column of the matrices $\mb G$ or $\mb V$ corresponds $\mb g_i$. Precisely, if $\mb g_i$ is a line $\mb a$ from $\mb A$, we set $j(i)=j$ where $j$ is the index of the column from $\mb G$ corresponding to $\mb a$. Otherwise, if $j$ is the index of the column from $\mb V$ corresponding to $\mb g_i$, we set $j(i) = m+j$.

Next, we argue that $\mb Y$ is a deterministic function of $\mathsf{Message}$, the matrices $\mb G$, $\mb V$ and the vector of indices $\mb j$ and $\mb c$. First, $\mb c$ provides the scalar responses directly. For the $d$-dimensional component of the responses, first, note that from $\mb G$, $\mb V$ and $\mb j$ one can easily recover the vectors $\mb g_1,\ldots, \mb g_m$. Next, using the algorithm for the second section of the Orthogonal Vector Game with Hints set with initial memory $\mathsf{Message}$ and the vectors $\mb g_1,\ldots, \mb g_m$ as responses of the oracle, one can inductively compute the queries $\mb x_1,\ldots, \mb x_m$. Last, $\mb Y$ is a deterministic function of $\mb x_i,\mb g_i, i\in[m]$ and $\mathsf{Message}$. This ends the claim that there is a function $\phi$ such that $\mb Y = \phi(\mathsf{Message}, \mb G, \mb V,\mb j, \mb c)$. Now by the data processing inequality,
\begin{equation}\label{eq:memory_constraint}
    I(\mb A';\mb Y\mid \mb G, \mb V,\mb j,\mb c) \leq I(\mb A';\mathsf{Message}\mid \mb G, \mb V,\mb j,\mb c) \leq H(\mathsf{Message}\mid \mb G, \mb V,\mb j, \mb c) \leq M.
\end{equation}
In the last inequality we used the fact that $\mathsf{Message}$ uses at most $M$ bits. Now, we have that 
\begin{equation}\label{eq:information_key}
    I(\mb A';\mb Y\mid \mb G, \mb V,\mb j,\mb c) = H(\mb A'\mid \mb G, \mb V,\mb j,\mb c) - H(\mb A'\mid \mb Y,\mb G, \mb V,\mb j,\mb c).
\end{equation}
In the next steps we bound the two terms. We start with the second term of the right hand side of Eq~\eqref{eq:information_key} using similar arguments to the proof given in \cite{marsden2022efficient}. Let $\Ecal$ be the event when the Player succeeds at Game \ref{game:hint_game}. Now consider the case when $\mb Y$ is a winning matrix. Then we have $\|\mb A\mb y_i\|_\infty\leq \alpha$ for all $i\leq k$. As a result, any line $\mb a$ of $\mb A'$ satisfies $\|\mb Y^\top \mb a\|_\infty \leq \alpha$. Further, we have that $\|P_{Span(\mb y_j,j<i)^\perp}(\mb y_i)\|\leq \beta$ for all $i\leq k$. By Lemma \ref{lemma:gram-schmidt_marsden}, there exist $\lceil k/5\rceil$ orthonormal vectors $\mb Z = [\mb z_1,\ldots,\mb z_{\lceil k/5\rceil}]$ such that for any $\mb x\in\Rbb^d$ one has $\|\mb Z^\top \mb x\|_\infty \leq \left(\frac{\sqrt d}{\beta}\right)^{5/4}\|\mb Y^\top \mb x\|_\infty$. In particular, all lines $\mb a$ of $\mb A'$ satisfy
\begin{equation*}
    \|\mb Z^\top \mb a\|_\infty \leq \left(\frac{\sqrt d}{\beta}\right)^{5/4} \alpha \leq \frac{1}{2},
\end{equation*}
where we used the hypothesis in the parameters $\alpha$ and $\beta$. Now by Lemma \ref{lemma:sensitive_base_marsden}, one has
\begin{equation*}
    \left| \left\{\mb a\in \{\pm 1\}^d: \|\mb Z^\top \mb a\|_\infty \leq \frac{1}{2} \right\}\right| \leq 2^d \Pbb_{\mb h\sim\Ucal(\{\pm 1\}^d)}\left( \|\mb Z^\top \mb h\|_\infty \leq \frac{1}{2} \right) \leq 2^{d-c_H\lceil k/5\rceil}.
\end{equation*}
Therefore, we proved that if $\mb Y'$ is a winning vector, $H(\mb A'\mid \mb Y =  \mb Y')\leq (n-m)(d-c_Hk/5)$. Otherwise, if $\mb Y'$ loses, we can directly use $H(\mb A'\mid \mb Y=\mb Y')\leq (n-m)d$. Combining these equations gives
\begin{align*}
    H(\mb A'\mid \mb Y,\mb G, \mb V, \mb j,\mb c) &\leq H(\mb A'\mid \mb Y) \\
    &\leq \Pbb(\Ecal^c)(n-m)d + \Pbb(\Ecal)  (n-m)(d-c_Hk/5)\\
    &\leq (n-m)(d-\Pbb(\Ecal)c_H k/5).
\end{align*}
Next, we turn to the first term of the right-hand side of Eq~\eqref{eq:information_key}.
\begin{align*}
    H(\mb A'\mid \mb G, \mb V, \mb j,\mb c) = H(\mb A\mid \mb G,\mb V,\mb j,\mb c)    &= H(\mb A\mid \mb V) -I(\mb A;\mb G,\mb j,\mb c\mid \mb V)\\
    &\geq H(\mb A\mid \mb V) - H(\mb G, \mb j,\mb c)\\
    &\geq H(\mb A\mid \mb V) - md -m\log(2m)-m\log(d^2)\\
    &= H(\mb A)- I(\mb A;\mb V) - md -3m\log(2d) \\
    &= (n-m)d -3m\log(2d) - I(\mb A;\mb V).
\end{align*}
In the second inequality, we use the fact that $\mb G$ uses $md$ bits and $\mb j$ can be stored with $m\log(2m)$ bits. Now by the chain rule,
\begin{equation*}
    I(\mb A;\mb V) = \sum_{i\leq m} I(\mb A;\mb w_i\mid \mb w_1,\ldots,\mb w_{i-1}).
\end{equation*}
Now if $\mb w_i = \mb v_l$, recalling that the vectors $\mb w_{i'}=\mb v_{l'}$ are ordered by increasing index of $l'$, we have
\begin{align*}
    I(\mb A;\mb w_i\mid \mb w_1,\ldots,\mb w_{i-1}) &= H(\mb w_i\mid \mb w_1,\ldots,\mb w_{i-1}) - H(\mb w_i \mid \mb A, \mb w_1,\ldots,\mb w_i)\\
    &\leq H(\mb w_i) - H(\mb w_i\mid \mb A, \mb w_1,\ldots,\mb w_i, \mb x_{l,1},\ldots, \mb x_{l,r_l})\\
    &= \log |\Dcal_\delta| - H(\mb w_i\mid \mb x_{l,1},\ldots, \mb x_{l,r_l}).
\end{align*}
In the last equality, we used the fact that if $\mb b_{l,1},\ldots, \mb b_{l,r_l}$ are the resulting vectors from the Gram-Schmidt decomposition of $\mb x_{l,1},\ldots,\mb x_{l,r_l}$, $\mb y_l$ is generated uniformly in $S^{d-1}\cap \{\mb y:\forall r\leq r_l,|\mb b_{l,r}^\top \mb y| \leq d^{-3}\}$ independently from the past history, and $\mb v_l = \phi_\delta(\mb y_l)$. Now by Lemma \ref{lemma:anti-concentration}, we know that
\begin{equation*}
    \Pbb_{\mb z\sim\Ucal(S^{d-1})}\left(\forall r\leq r_l,|\mb b_{l,r}^\top \mb z| \leq d^{-3}\right) \geq \frac{1}{e^{d^{-4}}d^{3k}}.
\end{equation*}
As a result, for any $\mb b_j(\delta)\in \Dcal_\delta$, one has
\begin{equation*}
    \Pbb(\mb w_i = \mb b_j(\delta)\mid \mb x_{l,1},\ldots, \mb x_{l,r_l}) \leq  \frac{\Pbb_{\mb z\sim\Ucal(S^{d-1})}(\mb z\in V_j(\delta))}{\Pbb_{\mb z\sim\Ucal(S^{d-1})}\left(\forall r\leq r_l,|\mb b_{l,r}^\top \mb z| \leq d^{-3}\right)}  \leq \frac{e^{d^{-4}}d^{3k}}{|\Dcal_\delta|},
\end{equation*}
where we used the fact that each cell has the same area. In particular, this shows that
\begin{equation*}
    H(\mb w_i\mid \mb x_{l,1},\ldots, \mb x_{l,r_l}) = \Ebb_{\mb b\sim \mb w_i\mid \mb x_{l,1},\ldots, \mb x_{l,r_l}}[-\log p_{\mb w_i\mid \mb x_{l,1},\ldots, \mb x_{l,r_l}}(\mb b)] \geq \log\left( \frac{|\Dcal_\delta|}{e^{d^{-4}}d^{3k}} \right).
\end{equation*}
Hence,
\begin{equation*}
    I(\mb A;\mb w_i\mid \mb w_1,\ldots,\mb w_{i-1}) \leq  3k\log d + d^{-4} \log e.
\end{equation*}
Putting everything together gives
\begin{align*}
    I(\mb A';\mb Y \mid \mb G, \mb V, \mb j) &\geq (n-m)d -3m\log(2d) - 3km\log d - 2md^{-4} - (n-m)(d- \Pbb(\Ecal)c_H k/5)\\
    &\geq \frac{c_H}{10}k (n-m) - 3km\log d - 1-3d\log(2d),
\end{align*}
where in the last equation we used $d\geq 2$.
Together with Eq~\eqref{eq:memory_constraint}, this implies
\begin{equation*}
    m\geq \frac{c_H kn/10-M-1 - 3d\log(2d)}{k(3\log d + c_H /10)}.
\end{equation*}
As a result, since $k\geq 20\frac{M+3d\log(2d)+1}{c_H n}$ and $n\geq d/4$, we obtain
\begin{equation*}
    m\geq \frac{c_H n}{60\log d + 2c_H} \geq \frac{c_H}{8(30\log d + c_H)}d.
\end{equation*}
This ends the proof of the proposition.
\end{proof}

We are now ready to prove the main result.

\begin{proof}[of Theorem \ref{thm:main_opt}]
    We set $n=\lceil d/4\rceil$ and $k=\lceil 20\frac{M+3d\log(2d)+1}{c_H n} \rceil$. By Proposition \ref{prop:true_first_order_procedure}, with probability at least $1-C\sqrt{\log d}/d^2$, the procedure is consistent with a first-order oracle for convex optimization. Hence, since the functions $F_{\mb A, \mb v ,P, L}$ are $\sqrt  d$-Lipschitz, any $M$-bit algorithm guaranteed to solve convex optimization within accuracy $\epsilon =\eta/(2d)=1/d^4$ for $1$-Lipschitz functions, yields an algorithm that is successful for the optimization procedure with probability at least $1-C\sqrt{\log d}/d^2$ and precision $\epsilon\sqrt d = \eta/(2\sqrt d)$. Suppose that it uses at most $Q$ queries. Then, by Proposition \ref{prop:reductionto_orthogonal_vector_game}, there is a strategy for Game \ref{game:hint_game} for parameters $(d,k,\lceil Q/p_{max}\rceil +1, M,\alpha = \frac{2\eta}{\gamma_1}, \beta = \frac{\gamma_2}{4})$ in which the Player wins with probability at least $1-C'\sqrt{\log d}/d$. Now for $d$ large enough, this probability is at least $1/2$. Further,
    \begin{equation*}
        \frac{2\eta}{\gamma_1}\left(\frac{4\sqrt d}{\gamma_2}\right)^{5/4} \leq \frac{(4/3)^{5/4}}{6}\eta d^3 \leq \frac{1}{2}.
    \end{equation*}
    Hence, by Proposition \ref{prop:lower_bound_queries_game}, one has
    \begin{equation*}
        \lceil Q/p_{max}\rceil +1 \geq \frac{c_H}{8(30\log d + c_H)}d.
    \end{equation*}
    Because $p_{max} =\Theta((d/k)^{1/3} \log^{-2/3} d)$, this implies
    \begin{equation*}
        Q = \Omega\left(\frac{(d/k)^{1/3}d}{\log^{5/3} d}\right) = \Omega\left( \frac{d^{5/3}}{(M+\log d)^{1/3} \log^{5/3} d}\right).
    \end{equation*}
    In particular, if $M=d^{1+\delta}$ for $\delta\in[0,1]$, the number of queries is $Q=\tilde\Omega(d^{1+(1-\delta)/3})$.
\end{proof}

\section{Memory-constrained feasibility problem}
\label{sec:feasibility}

\subsection{Defining the feasibility procedure}

Similarly to Section \ref{sec:optimization}, we pose $n=\lceil d/4\rceil$. Also, for any matrix $\mb A\in\{\pm 1\}^{n\times d}$, we use the same functions $\mb g_{\mb A}$ and $\tilde{\mb g}_{\mb A}$. We use similar techniques as those we introduced for the optimization problem. However, since in this case, the separation oracle only returns a separating hyperplane, without any value considerations of an underlying function, Procedure \ref{proc:optimization} can be drastically simplified, which leads to improved lower bounds.

Let $\eta_0 = 1/(24d^2)$, $\eta_1=\frac{1}{2\sqrt{d}}$, $\delta=1/d^3$, and $k\leq d/3-n$ be a parameter. Last, let $p_{max} = \lfloor (c_{d,1} d-1)/(k-1)\rfloor$, where $c_{d,1}$ is the same quantity as in Eq~\eqref{eq:definition_p_max}. The feasibility procedure is defined in Procedure \ref{proc:feasibility}. The oracle first randomly samples $\mb A\sim\Ucal( \{\pm 1\}^{n\times d})$ and $\mb v_0\sim \Ucal(\Dcal_\delta)$. This matrix and vector are then fixed in the rest of the procedure. Whenever the player queries a point $\mb x$ such that $\|\mb A\mb x\|_\infty > \eta_0$ (resp. $\mb v_0^\top \mb x>-\eta_1$), the oracle returns $\tilde{\mb g}_{\mb A}(\mb x)$ (resp. $\mb v_0$). All other queries are called \emph{informative} queries. With this definition, it now remains to define the separation oracle on informative queries. The oracle proceeds by periods in which the behavior is different. In each period $p$, the oracle constructs vectors $\mb v_{p,1},\ldots,\mb v_{p,k-1}$ inductively and keeps in memory some queries $i_{p,1},\ldots, i_{p,k}$ that will be called \emph{exploratory}. The first informative query $t$ will be the first exploratory query and starts period $1$.

Given a new query $\mb x_t$,
\begin{enumerate}
    \item If $\|\mb A\mb x\|_\infty > \eta_0$, the oracle returns $\tilde{\mb g}_{\mb A}(\mb x_t)$.
    \item If $\mb v_0^\top \mb x_t > -\eta_1$, the oracle returns $\mb v_0$.
    \item If $\mb x_t$ was queried in the past sequence, the oracle returns the same vector that was returned previously.
    \item Otherwise, let $p$ be the index of the current period and let $\mb v_{p,1},\ldots, \mb v_{p,l}$ be the vectors from the current period constructed so far, together with their corresponding exploratory queries $i_{p,1}\ldots, i_{p,l} < t$. Potentially, if $p=1$ one may not have defined any such vectors at the beginning of time $t$. In this case, let $l=0$.
    \begin{enumerate}
        \item If $\max_{1\leq l'\leq l} \mb v_{p,l'}^\top \mb x_t > - \eta_1$ (with the convention $\max_{\emptyset} = -\infty$), the oracle returns $\mb v_{p,l'}$ where $l' =\argmax_{l\leq r} \mb v_{p,l}^\top \mb x_t $. Ties are broken alphabetically.
        \item Otherwise, if $l<k-1$, we first define $i_{p,l+1} = t$. Then, let $\mb b_{p,1},\ldots,\mb b_{p,l+1}$ be the result from the Gram-Schmidt decomposition of $\mb x_{i_{p,1}},\ldots,\mb x_{i_{p,l+1}}$ and let $\mb y_{p,l+1}$ be a sample of the distribution obtained by the uniform distribution $\mb y_{p,l+1}\sim \Ucal(S^{d-1}\cap \left\{\mb z\in \Rbb^d : |\mb b_{p,r}^\top \mb z|\leq \frac{1}{d^3},\forall r\leq l+1\right\})$. We then pose $\mb v_{p,l+1} = \phi_\delta(\mb y_{p,l+1})$. Having defined this new vector, the oracle returns $\mb v_{p,l}$. We then increment $l$.
        \item Otherwise, if $r=k$, we define $i_{p,k} = i_{p+1,1} = t$. If $p+1\leq p_{max}$, this starts the next period $p+1$. As above, let $\mb b_{p+1,1}$ be the result of the Gram-Schmidt decomposition of $\mb x_{i_{p+1,1}}$ and sample $\mb y_{p+1,1}$ according to a uniform $\mb y_{p+1,1}\sim\Ucal(S^{d-1} \cap \left\{\mb z\in \Rbb^d : |\mb b_{p+1,1}^\top \mb z|\leq \frac{1}{d^3} \right\})$. We then pose $\mb v_{p+1,1} = \phi_\delta(\mb y_{p+1,1})$ and the oracle returns $\mb v_{p+1,1}$. We can then increment $p$ and reset $l=1$.
    \end{enumerate}
\end{enumerate}
The above construction ends when the period $p_{max}$ is finished. At this point, the oracle has defined the vectors $\mb v_{p,l}$ for all $p\leq p_{max}$ and $l\leq k$. We then define the successful set as
\begin{equation*}
    Q_{\mb A, \mb v} = \left\{\mb x\in B_d(0,1): \|\mb A\mb x\|_\infty \leq \eta_0, \mb v_0^\top \mb x \leq -\eta_1, \max_{p\leq p_{max},l\leq k-1} \mb v_{p,l}^\top \mb x \leq -\eta_1 \right\}.
\end{equation*}
From now on, the procedure uses any separation oracle for $Q_{\mb A, \mb v}$ as responses to the algorithm, while making sure to be consistent with previous oracle reponses if a query is exactly duplicated. We now define what we mean by solving the above feasibility procedure.

\begin{proc}[ht]
        
\caption{The feasibility procedure for algorithm $alg$}\label{proc:feasibility}

\SetAlgoLined
\LinesNumbered

\everypar={\nl}

\hrule height\algoheightrule\kern3pt\relax
\KwIn{$d$, $k$, $p_{max}$, algorithm $alg$}
\vspace{5pt}

Sample $\mb A\sim\Ucal(\{\pm 1\}^{n\times d})$ and $\mb v_0\sim\Ucal(\Dcal_\delta)$.\;

Initialize the memory of $alg$ to $\mb 0$ and let $p=1$, $l=0$.\;

\For{$t\geq 1$}{
    Run $alg$ with current memory to obtain a query $\mb x_t$\;
        
    \lIf{$\|\mb A \mb x_t\|>\eta_0$}{
        \Return $\tilde{\mb g}_{\mb A}(\mb x_t)$ as response to $alg$
    }
    \lElseIf{$\mb v_0^\top \mb x_t >-\eta_1$}{
        \Return $\mb v_0$ as response to $alg$
    }
    \lElseIf{Query $\mb x_t$ was made in the past}{
            \Return the same vector that was returned for $\mb x_t$
        }
    \uElse{
        \uIf{$\max_{1\leq l'\leq l} \mb v_{p,l'}^\top \mb x_t > - \eta_1$}{
            \Return $\mb v_{p,l'}$ where $l' =\argmax_{l\leq r} \mb v_{p,l}^\top \mb x_t $.
        }
        \uElseIf{$l<k-1$}{
            Let $i_{p,l+1}=t$ and compute Gram-Schmidt decomposition $\mb b_{p,1},\ldots,\mb b_{p,l+1}$ of $\mb x_{i_{p,1}},\ldots, \mb x_{i_{p,l+1}}$.\;

            Sample $\mb y_{p,l+1}$ uniformly on $\mathcal S^{d-1} \cap \{\mb z\in \Rbb^d : |\mb b_{p,l'}^\top \mb z| \leq d^{-3},\forall l'\leq l+1\}$ and define $\mb v_{p,l+1}=\phi_\delta(\mb y_{p,l+1})$.\;

            \Return $\mb v_{p,l+1}$ as response to $alg$ and increment $l\leftarrow l+1$.
        }
        \uElseIf{$p+1\leq p_{max}$}{
            Set $i_{p,k}=i_{p+1,1}=t$ and compute the Gram-Schmidt decomposition $\mb b_{p+1,1}$ of $\mb x_{i_{p+1,1}}$\;

            Sample $\mb y_{p+1,1}$ uniformly on $\mathcal S^{d-1} \cap \{ \mb z \in \Rbb^d : |\mb b_{p+1,1}^\top \mb z| \leq d^{-3}\}$ and define $\mb v_{p+1,1}=\phi_\delta(\mb y_{p+1,1})$.

            \Return $\mb v_{p+1,1}$ as response to $alg$, increment $p\leftarrow p+1$ and reset $l=1$.
        }
        \lElse{Set $i_{p_{max},k}=t$ and break the \textbf{for} loop}
    }
}

\vspace{5pt}

\lFor{$t'\geq t$}{
    Use any separation oracle for $Q_{\mb A, \mb v}$ consistent with previous responses
}

\hrule height\algoheightrule\kern3pt\relax
    \end{proc}

\begin{definition}
Let $alg$ be an algorithm for the feasibility problem. When running $alg$ with the responses of the feasibility procedure, we denote by $\mb v$ the set of constructed vectors and $\mb x^\star(alg)$ the final answer returned by $alg$. We say that an algorithm $alg$ is successful for the feasibility procedure with probability $q\in[0,1]$, if taking $\mb A\sim\Ucal(\{\pm 1\}^{n\times d})$, with probability at least $q$ over the randomness of $\mb A$ and of the procedure, $\mb x^\star(alg) \in Q_{\mb A, \mb v}.$
\end{definition}

In the rest of this section, we first relate this feasibility procedure to the standard feasibility problem, then prove query lower bounds to solve the feasibility procedure.

\subsection{Reduction from the feasibility problem to the feasibility procedure}

In the next proposition, we check that the above procedure indeed corresponds to a valid feasibility problem.

\begin{proposition}\label{prop:valid_problem}
    On an event of probability at least $1-C\sqrt{\log d}/d$, the procedure described above is a valid feasibility problem. More precisely, the following hold.
    \begin{itemize}
        \item There exists $\bar{\mb x} \in B_d(0,1)$ such that $\|\mb A\bar{\mb x}\|_\infty = 0$, $\mb v_0^\top \bar{\mb x}\leq -4\eta_1 $, and 
        \begin{equation*}
            \max_{p\leq p_{max},l\leq k-1} \mb v_{p,l}^\top \bar{\mb x} \leq - 4 \eta_1.
        \end{equation*}
        \item Let $\epsilon = \min\{\eta_0/\sqrt d,\eta_1\}/2$. Then, $ B_d\left(\bar{\mb x}-\epsilon \frac{\bar{\mb x}}{\|\bar{\mb x}\|} ,\epsilon \right )\subseteq B_d(0,1)\cap B_d(\bar{\mb x},2\epsilon)\subseteq Q_{\mb A, \mb v}.$
        \item Throughout the run of the feasibility problem, the separation oracle always returned a valid cut, i.e., for any iteration $t$, if $\mb x_t$ denotes the query and $\mb g_t$ is the returned vector from the oracle, one has
        \begin{equation*}
            \forall \mb x\in Q_{\mb A, \mb v}, \quad \langle \mb g_t,\mb x_t - \mb x\rangle > 0.
        \end{equation*}
        Further, responses are consistent: if $\mb x_t=\mb x_{t'}$, the responses of the procedure at times $t$ and $t'$ coincide.
    \end{itemize}
\end{proposition}

We use a similar proof to that of Proposition \ref{prop:low_optimum}.

\begin{proof}
    For convenience, we rename $\mb v_{p,l} = \mb v_{(p-1)(k-1)+l}$. Also, let $l_{max} = p_{max}(k-1) \leq c_{d,1} d-1$. Next, let $C_d = \sqrt{40l_{max}\log d}$. We define the vector 
    \begin{equation*}
        \bar {\mb x} = - \frac{1}{C_d} \sum_{l=0}^{l_{max}} P_{Span(\mb a_i,i\leq n)^\perp}(\mb v_l).
    \end{equation*}
    Since $l_{max}\leq p_{max}(k-1)\leq  c_{d,1}d-1$, the same arguments as in the proof of Proposition \ref{prop:low_optimum} show that on an event $\Ecal$ of probability at least $1-C\sqrt{\log d}/d$, we have $\|\bar{\mb x}\| \leq 1$ and
    \begin{equation*}
        \max_{0\leq l\leq l_{max}}\mb v_l^\top \bar{\mb x} \leq -\frac{1}{40\sqrt{(l_{max}+1)\log d}}\leq -\frac{2}{\sqrt d}=-4\eta_1,
    \end{equation*}
    where in the second inequality we used $l_{max} \leq c_{d,1}d-1$. Now by construction, one has $\|\mb A\bar{\mb x}\|_\infty = 0$. This ends the proof of the first claim of the proposition. We now turn to the second claim, which is immediate from the fact that $\mb x\mapsto \|\mb A\mb x\|_\infty$ is $\sqrt d$-Lipschitz and both $\mb x\mapsto \mb v_0^\top \mb x$ and  $\mb x\mapsto \max_{p\leq p_{max},l\leq k}\mb v_{p,l}^\top \mb x$ are $1$-Lipschitz. Therefore, $B_d(\bar{\mb x}-\epsilon \bar{\mb x}/\|\bar{\mb x}\|,\epsilon)\subseteq B_d(0,1)\cap B_d(\bar{\mb x},2\epsilon)\subset Q_{\mb A, \mb v}$. It now remains to check that the third claim is satisfied. It suffices to check that this is the case during the construction phase of the feasibility procedure. By construction of $Q_{\mb A,\mb v}\subset \{\mb x: \|\mb A\mb x\|_\infty \leq \eta_0\}$.
    
    Hence, it suffices to check that for informative queries $\mb x_t$, the returned vectors $\mb g_t$ are valid separation hyperplanes. By construction, these can only be either $\mb v_0$ or $\mb v_{p,l}$ for $p\leq p_{max}$, $l\leq k-1$. We denote by $\mb w$ this vector. Let $t'$ be the first time $\mb x_t$ was queried. There are two cases. Either $\mb w$ was not constructed at time $t'$, in which case, by construction this means that we are in scenario (2) or (4a). Both cases imply $\mb w^\top \mb x_t > -\eta_1$. Hence, $\mb w$ which is returned by the procedure is a valid separation hyperplane. Now suppose that $\mb w=\mb v_{p,l}$ was constructed at time $t'$---scenarios (4b) or (4c). By construction, one has $|\mb b_{p,r}^\top \mb y_{p,l}|\leq d^{-3}$ for all $r\leq l$. Decomposing $\mb x_t = \mb x_{i_{p,l}} = \alpha \mb b_{p,1}+\ldots + \alpha_l \mb b_{p,l}$, we obtain
    \begin{equation*}
        |\mb x_t^\top \mb y_{p,l}| \leq \frac{\|\mb \alpha\|_1}{d^3} \leq \frac{1}{d^2\sqrt d}.
    \end{equation*}
    As a result, $\mb y_{p,l}^\top \mb x_t \geq -1/(d^2\sqrt d)$. Now because $\mb v_{p,l} = \phi_\delta(\mb y_{p,l})$, we have $\|\mb v_{p,l} - \mb y_{p,l}\|\leq \delta$. Hence, for any $d\geq 2$,
    \begin{equation*}
        \mb w^\top \mb x_t \geq -1/(d^2\sqrt d)-\delta >-\eta_1.
    \end{equation*}
    Hence, $\mb w$ was a valid separation hyperplane. The last claim that the responses of the procedure are consistent over time is a direct consequence from its construction. This ends the proof of the proposition.
\end{proof}

As a simple consequence of this result, solving the feasibility problem is harder than solving the feasibility procedure with high probability.

\begin{proposition}\label{prop:reduction_feasibility_procedure}
    Let $alg$ be an algorithm that solves the feasibility problem with accuracy $\epsilon = 1/(48d^2\sqrt d)$. Then, it solves the feasibility procedure with probability at least $1-C\sqrt{\log d}/d$.
\end{proposition}

\begin{proof}
    Let $\Ecal$ be the event of probability at least $1-C\sqrt{\log d}/d$ defined in Proposition \ref{prop:valid_problem}. We show that on $\Ecal$, $alg$ solves the feasibility procedure. On $\Ecal$, the feasibility procedure emulates is a valid feasibility oracle. Further, on $\Ecal$, the successful set contains a closed ball of radius $\epsilon$. As a result, on $\Ecal$, $alg$ finds a solution to the feasibility problem emulated by the procedure.
\end{proof}

Next, we show that it is necessary to finish the $p_{max}$ periods to solve the feasibility procedure.

\begin{proposition}\label{prop:finish_procedure_necessary}
    Fix an algorithm $alg$. Then, if $\Acal$ denotes the event when $alg$ succeeds and $\Bcal$ denotes the event when the procedure ends period $p_{max}$ with $alg$, then $\Ecal\subseteq\Bcal$.
\end{proposition}

\begin{proof}
    Consider the case when the period $p_{max}$ was not ended. Let $\mb x^\star$ denote the last query performed by $alg$. We consider the scenario in which $\mb x^\star$ fell. Let $t$ be the first time when $alg$ submitted query $\mb x^\star$. For any of the scenarios (1), (2), or (4a), by construction of $Q_{\mb A, \mb v}$, we already have $\mb x_t\notin Q_{\mb A, \mb v}$. It remains to check scenarios (4b) and (4c) for which the procedure constructs a new vector $\mb v_{p,l}$, where $p$ is the index of the period of $t$ and $i_{p,1},\ldots,i_{p,l}=t$ are the previous exploratory queries in period $p$. We decompose $\mb x_t = \mb x_{i_{p,l}} = \alpha_1\mb b_{p,1}+\alpha_l\mb b_{p,l}$. Now by construction,
    \begin{equation*}
        |\mb x_t^\top \mb y_{p,l}| = |\mb x_{i_{p,l}}^\top \mb y_{p,l}| \leq \frac{\|\mb \alpha\|_1}{d^3} \leq \frac{1}{d^2\sqrt d}.
    \end{equation*}
    As a result, $\mb x_t^\top \mb v_{p,l} \geq -|\mb x_t^\top \mb y_{p,l}|-\delta \geq -d^{-2.5} - d^{-3} >-\eta_1$, for any $d\geq 2$. Thus, $\mb x_t = \mb x^\star\notin Q_{\mb A, \mb v}$. This shows that in order to succeed at the feasibility procedure, an algorithm needs to end all $p_{max}$ periods.
\end{proof}

\subsection{Reduction to the Orthogonal Vector Game with Hints.}

The remaining piece of our argument is to show that solving the feasibility procedure is harder than solving the Orthogonal Vector Game with Hints, Game \ref{game:hint_game}.

\begin{proposition}\label{prop:reduction_orthogonal_game}
    Let $\mb A\sim \Ucal(\{\pm 1\}^{n\times d})$. If there exists an $M$-bit algorithm that solves the feasibility problem described above using $m p_{max}$ queries with probability at least $q$ over the randomness of the algorithm, choice of $\mb A$ and the randomness of the separation oracle, then there is an algorithm for Game \ref{game:hint_game} for parameters $(d,k,m,M,\alpha = \frac{\eta_0}{\eta_1},\beta = \frac{\eta_1}{2})$, for which the Player wins with probability at least $q$ over the randomness of the player's strategy and $\mb A$.
\end{proposition}

\begin{proof}
    Let $alg$ be an $M$-bit algorithm solving the feasibility problem with $m p_{max}$ queries with probability at least $q$. In Algorithm \ref{alg:strategy_feasibility}, we describe the strategy of the player in Game \ref{game:hint_game}.

\begin{algorithm}[ht]
        
\caption{Strategy of the Player for the Orthogonal Vector Game with Hints}\label{alg:strategy_feasibility}

\setcounter{AlgoLine}{0}
\SetAlgoLined
\LinesNumbered

\everypar={\nl}

\hrule height\algoheightrule\kern3pt\relax
\KwIn{$d$, $k$, $p_{max}$, $m$, algorithm $alg$}

\vspace{5pt}

{\nonl \textbf{Part 1:}} Strategy to store $\mathsf{Message}$ knowing $\mb A$\;

Initialize the memory of $alg$ to be $\mb 0$.\;

Submit $\emptyset$ to the Oracle and use the response as $\mb v_0$.\;

Run $alg$ with the optimization procedure knowing $\mb A$ and $\mb v_0$ until the first exploratory query $\mb x_{i_{1,1}}$.

\For{$p\in[p_{max}]$}{
    Let $\mathsf{Memory}_p$ be the current memory state of $alg$ and $i_{p,1}$ the current iteration step. \;
    
    Run $alg$ with the feasibility procedure until period $p$ ends at iteration step $i_{p+1,1}$. If $alg$ stopped before, \Return the strategy fails. When needed to sample a unit vector $\mb v_{p',l'}$, submit vectors $\mb x_{i_{p',1}},\ldots \mb x_{i_{p',l'}}$ to the Oracle. We use the corresponding response of the Oracle as $\mb v_{p',l'}$.\;

    \uIf{$i_{p+1,1}-i_{p,1}\leq m$}{
        Set $\mathsf{Message}=\mathsf{Memory}_p$
    }
}
\lFor{Remaining queries to perform to Oracle}{
    Submit arbitrary query, e.g. $\emptyset$
}

\lIf{$\mathsf{Message}$ has not been defined yet}{\Return The strategy fails}

Submit $\tilde{\mb g}_{\mb A, \mb v}$ to the Oracle as defined in Eq~\eqref{eq:definition_subgradient_feasibility}.\;

\vspace{5pt}

{\nonl \textbf{Part 2:}} Strategy to make queries\;

Set the memory state of $alg$ to be $\mathsf{Message}$.\;

\For{$i\in[m]$}{
    Run $alg$ with current memory to obtain a query $\mb z_i$.\;
    
    Submit $\mb z_i$ to the Oracle from Game \ref{game:hint_game}, to get response $(\mb g_i,s_i)$.\;
    
    Compute $\tilde{\mb g}_i$ using $\mb z_i$, $\mb g_i$ and $s_i$ as defined in Eq~\eqref{eq:recover_subgradient_feasibility} and pass $\tilde{\mb g}_i$ as response to $alg$.\;
}

\vspace{5pt}

{\nonl \textbf{Part 3:}} Strategy to return vectors\;

\lFor{$l\in[k]$}{
    Set $i_l$ to be the index $i$ of the first query $\mb z_i$ for which $s_i=l$, if it exists
}

\uIf{index $i_k$ has not been defined yet}{
    With the current memory of $alg$ find a new query $\mb z_{m+1}$ and set $i_k=m+1$.\;
}
\Return $\left\{\frac{\mb z_{i_1}}{\|\mb z_{i_1}\|},\ldots, \frac{\mb z_{i_k}}{\|\mb z_{i_k}\|}\right\}$ to the Oracle.

\hrule height\algoheightrule\kern3pt\relax
\end{algorithm}

    In the first part of the strategy, the player observes $\mb A$. Then they proceed to simulate the feasibility problem with $alg$ using parameters $\mb A$. When needed to sample a vector $\mb v_{p,l}$ (resp. $\mb v_0$), the player submits the corresponding queries $\mb x_{i_{p,1}},\ldots,\mb x_{i_{p,l}}$ (resp. $\emptyset$) useful to define $\mb v_{p,l}$. The player then takes the response given by the Oracle as that vector $\mb v_{p,l}$ (resp. $\mb v_0$), which simulates exactly a run of the feasibility procedure. Further, since $1+p_{max}(k-1) \leq d$, the player does not run out of queries. Importantly, during the run, the player keeps track of the length $i_{p,k}-i_{p,1}$ of period $p$. The first time we encounter a period $p$ with length at most $m$, we set $\mathsf{Message} = \mathsf{Memory}_p$, the memory state of $alg$ at the beginning of period $p$. If there is no such period, the strategy fails. Also, if $alg$ stopped before ending period $p_{max}$, the strategy fails. Next, the algorithm submits the following function $\tilde {\mb g}_{\mb A, \mb v}$ to the Oracle. Since the responses of the feasibility procedure are consistent over time, we adopt the following notation. For a previously queried vector $\mb x$ of $alg$, we denote $\mb g(\mb x)$ the vector which was returned to $alg$ during the first part (lines 3-9 of Algorithm \ref{alg:strategy_feasibility}). 
    \begin{equation}\label{eq:definition_subgradient_feasibility}
        \tilde{\mb g}_{\mb A, \mb v} :\mb x\mapsto \begin{cases}
            (\mb 0, 1) &\text{if }\mb x \text{ was never queried in the first part},\\
            (\mb a_i,1) &\text{ow. and if }\mb g(\mb x)\in\{\pm \mb a_i\},i\leq n,\\
            (\mb v_0,2) &\text{ow. and if }\mb g(\mb x)=\mb v_0,\\
            (\mb v_{p',l'},2+l'\1_{p'=p}+k\1_{p'=p+1,l'=1}) &\text{ow. and if }\mb g(\mb x) = \mb v_{p',l'}, p'\leq p_{max} , l\leq k-1.
        \end{cases}
    \end{equation} 
    Intuitively, the first component of $\tilde{\mb g}$ gives the returned vector in the first period, at the exception that we always return $\mb a_i$ instead of $\{\pm \mb a_i\}$. The second term has values in $[2+k\leq d^2].$ Hence, the submitted function is valid.
    
    Next, in the second part of the algorithm, the player proceeds to simulate a run the feasibility procedure with $alg$ on period $p$. To do so, we first set the memory state of $alg$ to $\mathsf{Message}$. Each new query $\mb z_i$ is submitted to the Oracle of Game \ref{game:hint_game} to get a response $(\mb g_i,s_i)$. Then, we compute $\tilde{\mb g}_i$ as follows
    \begin{equation}\label{eq:recover_subgradient_feasibility}
        \tilde{\mb g}_i =\begin{cases}
            \mb g_i &\text{if }s_i\geq 2,\\
            sign(\mb g_i^\top \mb z_i)\mb g_i &\text{if }s_i=1.
        \end{cases}
    \end{equation}
    One can easily check that $\tilde{\mb g}_i$ corresponds exactly to the response that was passed to $alg$ in the first part of the strategy. The player then passes $\tilde{\mb g}_i$ to $alg$ so that it can update its state. We repeat this process for $m$ steps. Further, the player can also keep track of the exploratory queries: the index $i_l$ of the first response satisfying $s_i=2+l$ for $l\leq k-1$ (resp. $s_i=2+k$)is the exploratory query which led to the construction of $\mb v_{p,l}$ (resp. $\mb v_{p+1,1}$) in the first part. Last, we check if the last index $i_k$ was defined. If not, we pose $i_k=m+1$ and let $\mb z_{m+1}$ be the next query of $alg$ with the current memory. The player then returns the vectors $\frac{\mb z_{i_1}}{\|\mb z_{i_1}\|},\ldots,\frac{\mb z_{i_k}}{\|\mb z_{i_k}\|}$. This ends the description of the player's strategy.

    By Proposition \ref{prop:finish_procedure_necessary}, on an event $\Ecal$ of probability at least $q$, the algorithm $alg$ succeeds and ends period $p_{max}$. As a result, similarly as in the proof of Proposition \ref{prop:reductionto_orthogonal_vector_game}, since $alg$ makes at most $m p_{max}$ queries, and there are $p_{max}$ periods, there must be a period of length at most $m$. Hence the strategy never fails at this phase of the player's strategy on the event $\Ecal$. Further, we already checked that in the second phase, the vectors $\tilde{\mb g}_i$ passed to $alg$ coincide exactly with the responses passed to $alg$ in the first part. Thus, this shows that during the second part, the player simulates exactly the run of the feasibility problem on period $p$. More precisely, the queries coincide with the queries in the feasibility problem at times $i_{p,1},\ldots ,\min\{i_{p,k}, i_{p,1}+m-1\}$. Now because the first part succeeded on $\Ecal$, we have $i_{p,k}\leq i_{p,0}+m$. Therefore, if $i_k$ has not yet been defined, this means that we had $i_{p,k}=i_{p,1}+m$. Hence, the next query with the current memory $\mb z_{m+1}$ is exactly the query $\mb x_{i_{p,k}}$ for the feasibility problem. This shows that the vectors $\mb z_{i_1},\ldots,\mb z_{i_k}$ coincide exactly with the vectors $\mb x_{i_{p,1}},\ldots,\mb x_{i_{p,k}}$ when running $alg$ on the feasibility problem in the first part.
    
    We now show that the returned vectors are successful for Game \ref{game:hint_game}. By construction, $\mb x_{i_{p,1}},\ldots, \mb x_{i_{p,k}}$ are all informative. In particular, $\|\mb A\mb x_{i_{p,l}}\|_\infty \leq  \eta_0$ for all $1\leq l\leq k$. Further,  these queries did not fall in scenario (2), hence $\mb v_0^\top \mb x_{i_{p,l}} < -\eta_1$, which implies $\|\mb x_{i_{p,l}}\| >\eta_1$ for all $l\leq k$. As a result,
    \begin{equation*}
        \frac{\|\mb A\mb x_{i_{p,l}}\|_\infty}{\|\mb x_{i_{p,l}}\|} \leq \frac{\eta_0}{\eta_1}.
    \end{equation*}
    Next fix $l\leq k-1$. By construction of $\mb y_{p,l}$,
    \begin{equation*}
        \|P_{Span(\mb x_{i_{p,l'}},l'\leq l)}(\mb y_{p,l})\|^2 =\sum_{l'\leq l} |\mb b_{p,l'}^\top \mb y_{p,l}|^2 \leq \frac{k}{d^6} \leq \frac{1}{d^5}.
    \end{equation*}
    Hence,
    \begin{equation*}
        \|\mb v_{p,l} - P_{Span(\mb x_{i_{p,l'}},l'\leq l)^\perp}(\mb y_{p,l}) \| \leq \|P_{Span(\mb x_{i_{p,l'}},l'\leq l)}(\mb y_{p,l})\| +\delta \leq \frac{1}{d^5}+ \delta.
    \end{equation*}
    As a result, since $\mb x_{p,l+1}^\top \mb v_{p,l} <-\eta_1$, we have
    \begin{equation*}
        \|P_{Span(\mb x_{i_{p,l'}},l'\leq l)^\perp}(\mb x_{p,l+1})\| \geq |\mb x_{p,l+1}^\top P_{Span(\mb x_{i_{p,l'}},l'\leq l)^\perp}(\mb y_{p,l})\| > \eta_1 - \frac{1}{d^5}-\delta\geq \frac{\eta_1}{2}.
    \end{equation*}
    This shows that the returned vectors $\frac{\mb x_{i_{p,1}}}{\|\mb x_{i_{p,1}}\|},\ldots, \frac{\mb x_{i_{p,k}}}{\|\mb x_{i_{p,k}}\|}$ are successful for Game \ref{game:hint_game} with parameters $\alpha = \frac{\eta_0}{\eta_1}$ and $\beta = \frac{\eta_1}{2}$. This ends the proof that strategy succeeds on $\Ecal$ for these parameters, which ends the proof of the proposition.
\end{proof}

We are now ready to prove the main result.

\begin{proof}[of Theorem \ref{thm:main_feasibility}]
    Suppose that there is an algorithm $alg$ for solving the feasibility problem to optimality $\epsilon=1/(48d^2\sqrt d)$ with memory $M$ and at most $Q$ queries. Let $k=\lceil 20\frac{M+3d\log(2d)+1}{c_Hn}\rceil$. By Proposition \ref{prop:reduction_feasibility_procedure}, it solves the feasibility procedure with parameter $k$ with probability at least $1-C\sqrt{\log d}/d$. By Proposition \ref{prop:reduction_orthogonal_game} there is an algorithm for Game \ref{game:hint_game} that wins with probability $1/3$ with $m=\lceil Q/p_{max}\rceil$ and parameters $\alpha =\eta_0/\eta_1$ and $\beta=\eta_1/2$. Now we check that
    \begin{equation*}
        \alpha\left(\frac{\sqrt d}{\beta}\right)^{5/4} \leq 12d^2\eta_0 = \frac{1}{2}.
    \end{equation*}
    Hence, by Proposition \ref{prop:lower_bound_queries_game}, we have
    \begin{equation*}
        m\geq \frac{c_H}{8(30\log d+c_H)}d.
    \end{equation*}
    This shows that
    \begin{equation*}
        Q \geq \Omega\left(p_{max}\frac{d}{\log d}\right) = \Omega\left(\frac{d^2}{k\log^3 d} \right) = \Omega\left(\frac{d^3}{(M+\log d)\log^3 d}\right).
    \end{equation*}
    This implies that for a memory $M=d^{2-\delta}$ with $0\leq \delta\leq 1$ the number of queries is $Q=\tilde\Omega(d^{1+\delta})$.
\end{proof}

\acks{This work was partly funded by ONR grant N00014-18-1-2122 and AFOSR grant FA9550-19-1-0263.}

\bibliography{refs}

\appendix

\section{Concentration bounds}

The following result gives concentration bounds for the norm of the projection of a random unit vector onto linear subspaces.

\begin{proposition} \label{prop:projection_concentration}
Let $P$ be a projection in $\Rbb^d$ of rank $r$ and let $\mb x\in \Rbb^d$ be a random vector sampled uniformly on the unit sphere $\mb x\sim \Ucal(S^{d-1})$. Then, for every $t>0$,
\begin{equation*}
    \max\left\{\Pbb\left(\|P(\mb x)\|^2 - \frac{r}{d}\geq t\right) , \Pbb\left(\|P(\mb x)\|^2 - \frac{r}{d}\leq -t\right) \right\} \leq e^{-dt^2}.
\end{equation*}
Further, if $r=1$ and $d\geq 2$,
\begin{equation*}
    \Pbb\left(\|P(\mb x)\|\geq  \sqrt{\frac{t}{d-1}}\right) \leq 2\sqrt{t}e^{-t/2}.
\end{equation*}
\end{proposition}

\begin{proof}
    First, by isometry, we can assume that $P$ is the projection onto the coordinate vectors $\mb e_1,\ldots\mb e_r$. Then, let $\mb y\sim\Ncal(0,1)$ be a normal vector. Note that $\mb x = \frac{\mb y}{\|\mb y\|}\sim\Ucal(S^{d-1})$. Further,
    \begin{equation*}
        \|\mb x\|^2\geq \frac{r}{d}+t \iff \left(1-\frac{r}{d}-t\right)\sum_{i=1}^r y_i^2 \geq \left(\frac{r}{d}+t\right)\sum_{i=r+1}^d y_i^2.
    \end{equation*}
    Now note that $Z_1=\sum_{i=1}^r y_i^2$ and $Z_2=\sum_{i=r+1}^d y_i^2$ are two independent random chi squared variables of parameters $r$ and $d-r$ respectively. Recalling that the moment generating function of $Z\sim \chi^2(k)$ is $\Ebb[e^{s Z}]=(1-2s)^{-k/2}$ for $s<1/2$. Therefore, for any
    \begin{equation}\label{eq:correct_s}
        -\frac{1}{2(r/d+t)}<s<\frac{1}{2(1-r/d-t)},
    \end{equation}
    one has
    \begin{align*}
        \Pbb\left(\|P(\mb x)\|^2 - \frac{r}{d}\geq t\right) &\leq \Ebb\left[ \exp\left(s \left(1-\frac{r}{d}-t\right)Z_1 - s\left(\frac{r}{d}+t\right)Z_2\right)\right]\\
        &= \frac{\left[1-2s\left(1-\frac{r}{d}-t\right)\right]^{-r/2}}{\left[1-2s\left(\frac{r}{d}+t\right)\right]^{-(d-r)/2}}.
    \end{align*}
    Now let $s = \frac{1}{2} \left( \frac{1-r/d}{1-r/d-t} - \frac{r/d}{r/d+t} \right)$, which satisfies Eq~\eqref{eq:correct_s}. The previous equation readily yields
    \begin{equation*}
        \Pbb\left(\left|\|P(\mb x)\|^2 - \frac{r}{d}\right|\geq t\right) \leq \exp\left(-\frac{d}{2}d_{KL}\left(\frac{r}{d};\frac{r}{d}+t\right)\right)\leq e^{-dt^2}.
    \end{equation*}
    In the last inequality we used Pinsker's inequality $d_{KL}(r/d;r/d+t)\geq 2\delta(\Bcal(r/d),\Bcal(d/r+t))^2 = 2t^2$, where $\Bcal(q)$ is the Bernouilli distribution of parameter $q$. Replacing $P$ with $Id-P$ and $r$ with $d-r$ gives the other inequality
    \begin{equation*}
        \Pbb\left(\|P(\mb x)\|^2 - \frac{r}{d}\leq - t\right) \leq e^{-dt^2}.
    \end{equation*}
    This gives first claim. For the second claim, supposing that $r=1<d$, from the above equation, we have
    \begin{equation*}
        \Pbb\left(\|P(\mb x)\|^2 \geq  \frac{t}{d}\right) \leq \exp\left(-\frac{d}{2}d_{KL}\left(\frac{1}{d};\frac{t}{d}\right)\right) = \sqrt t \left(\frac{1-\frac{t}{d}}{1-\frac{1}{d}} \right)^{(d-1)/2} \leq \sqrt{2t} e^{-t(d-1)/(2d)}.
    \end{equation*}
    Thus,
    \begin{equation*}
        \Pbb\left(\|P(\mb x)\|^2 \geq  \frac{t}{d-1}\right) \leq \sqrt{\frac{2(d-1)}{d}}\sqrt t e^{-t/2},
    \end{equation*}
    which ends the proof of the proposition.
\end{proof}

Next, we need the following lemma which gives a concentration inequality for discretized samples in $\Dcal_d$ and approximately perpendicular to $k\leq d/3-1$ vectors. 

\begin{lemma}\label{lemma:concentration_bound}
    Let $0\leq k\leq d/3-1$ and $\mb x_1,\ldots, \mb x_k\in B_d(0,1)$ be $k$ orthonormal vectors in the unit ball, and $\mb x\in B_d(0,1)$. Denote by $\mu$ the distribution on the unit sphere corresponding to the uniform distribution $\mb y\sim \Ucal(S^{d-1}\cap \{ \mb w\in \Rbb^d : |\mb x_i^\top \mb w| \leq d^{-3},\forall i\leq k \})$. Let $\mb y\sim \mu$. Then, for $t\geq 2 $,
    \begin{equation*}
        \Pbb\left(|\mb x^\top \mb y| \geq \sqrt{\frac{t}{d}} + \frac{1}{d^2}\right) \leq 2\sqrt t e^{-t/3}.
    \end{equation*}
    Further, let $\delta\leq 1$ and $\mb z = \phi_\delta(\mb y)$. Then for $t\geq 4$, 
    \begin{equation*}
        \Pbb\left(|\mb x^\top \mb z| \geq \sqrt{\frac{t}{d}} + \frac{1}{d^2} +\delta \right) \leq 2\sqrt t e^{-t/3}.
    \end{equation*}
\end{lemma}

\begin{proof}
    We use the same notations as above and denote by $\Ecal=\{ |\mb x_i^\top \mb y| \leq d^{-3},\forall i\leq k\}$ the event considered and $\mb y\sim \mu$. We decompose $\mb y = \alpha_1 \mb x_1+\ldots + \alpha_k \mb x_k + \mb y'$, where $\mb y'\in Span(\mb x_i,i\leq k)^\perp:=E$. Now note that $\frac{\mb y'}{\|\mb y'\|}$ is a uniformly random unit vector in $E$. As a result, using Proposition \ref{prop:projection_concentration}, we obtain for any $t\geq 2$,
    \begin{align*}
        \Pbb\left(|\mb x^\top \mb y'| \geq \sqrt{\frac{t}{d-k-1}} \right) &= \Pbb\left(|P_E(\mb x)^\top \mb y'| \geq  \sqrt{\frac{t}{d-k-1}} \right) \\
        &\leq 2\sqrt{t}e^{-t/2}.
    \end{align*}
    Also, because by definition of $\mu$, we have $|\alpha_i|\leq d^{-3}$ for all $i\leq k$, we obtain $|\mb x^\top \mb y| \leq \frac{k}{d^3} + |\mb x^\top \mb y'| \leq \frac{1}{d^2} + |\mb x^\top \mb y'|$. As a result, using the fact that $d-k-1\geq 2d/3$, the previous equation shows that
    \begin{equation*}
        \Pbb\left(|\mb x^\top \mb y| \geq \sqrt{\frac{3t}{2d}} +\frac{1}{d^2}\right) \leq \Pbb\left(|\mb x^\top \mb y'| \geq \sqrt{\frac{t}{d-k-1}} \right)\leq 2\sqrt{t}e^{-t/2}.
    \end{equation*}
    Next, we use the fact that $\|\mb z - \mb y\| = \|\phi_\delta(\mb y) - \mb y\| \leq \delta$ to obtain
    \begin{equation*}
        \Pbb\left(|\mb x^\top \mb z| \geq \sqrt{\frac{t}{d}} +\frac{1}{d^2} +\delta\right) \leq \Pbb\left(|\mb x^\top \mb y| \geq \sqrt{\frac{t}{d}} +\frac{1}{d^2} \right)  \leq 2\sqrt{t}e^{-t/3}.
    \end{equation*}
    This ends the proof of the lemma.
\end{proof}



\section{An improved result on robustly-independent vectors}

The following lemma serves the same purpose as \cite[Lemma 34]{marsden2022efficient}. Namely, from successful vectors of the Game \ref{game:hint_game}, it allows to recover an orthonormal basis that is still approximately in the nullspace of $\mb A$. The following version gives a stronger version that improves the dependence in $d$ of our chosen parameters.

\begin{lemma}
\label{lemma:gram-schmidt_marsden}
    Let $\delta\in(0,1]$ and suppose that we have $r\leq d$ unit norm vectors $\mb y_1,\ldots,\mb y_r\in \Rbb^d$. Suppose that for any $i\leq k$,
    \begin{equation*}
        \|P_{Span(\mb y_j,j<i)^\perp}(\mb y_i)\|\geq \delta.
    \end{equation*}
    Let $\mb Y=[\mb y_1,\ldots,\mb y_r]$ and $s\geq 2$. There exists $\lceil r/s\rceil$ orthonormal vectors $\mb Z=[\mb z_1,\ldots,\mb z_{\lceil r/s \rceil}]$ such that for any $\mb a\in\Rbb^d$,
    \begin{equation*}
        \|\mb Z^\top \mb a\|_\infty\leq  \left(\frac{\sqrt d}{\delta}\right)^{s/(s-1)}\|\mb Y^\top\mb a\|_\infty.
    \end{equation*}
\end{lemma}

\begin{proof}
        Let $\mb B=(\mb b_1,\ldots, \mb b_r)$ be the orthonormal basis given by the Gram-Schmidt decomposition of $\mb y_1,\ldots, \mb y_r$. By definition of the Gram-Schmidt decomposition, we can write $\mb Y = \mb B\mb C$ where $\mb C$ is an upper-triangular matrix. Further, its diagonal is exactly $diag(\|P_{Span(\mb y_{l'},l'<l)^\perp}(\mb y_l)\|, l\leq r)$. Hence,
    \begin{equation*}
        \det(\mb Y) = \det(\mb C) = \prod_{l\leq r} \|P_{Span(\mb y_{l'},l'<l)^\perp}(\mb y_l)\| \geq \delta^r.
    \end{equation*}
    We now introduce the singular value decomposition $\mb Y = \mb U 
 diag(\sigma_1,\ldots,\sigma_r)\mb V^\top$, where $\mb U\in \Rbb^{d\times r}$ and $\mb V\in \Rbb^{r\times r}$ have orthonormal columns, and $\sigma_1\geq\ldots\geq \sigma_r$. Next, for any vector $\mb z\in\Rbb^d$, since the columns of $\mb Y$ have unit norm,
    \begin{equation*}
        \|\mb Y\mb z\|_2 \leq \sum_{l\leq r}|z_l| \|\mb y_l\|_2 \leq \|\mb z\|_1 \leq \sqrt d \|\mb z\|_2. 
    \end{equation*}
    In the last inequality we used Cauchy-Schwartz. Therefore, all singular values of $\mb Y$ are upper bounded by $\sigma_1\leq \sqrt d$. Thus, with $r' = \lceil r/s\rceil$
    \begin{equation*}
        \delta^r \leq \det(\mb Y) =\prod_{l=1}^r \sigma_l \leq d^{(r'-1)/2} \sigma_{r'}^{r-r'+1} \leq d^{r/2s}\sigma_{r'}^{(s-1)r/s},
    \end{equation*}
    so that $\sigma_{r'}\geq \delta^{s/(s-1)}/d^{1/(2s)}$. We are ready to define the new vectors. We pose for all $i\leq r'$, $\mb z_i = \mb u_i$ the $i$-th column of $\mb U$. These correspond to the $r'$ largest singular values of $\mb Y$ and are orthonormal by construction. Then, for any $i\leq r'$, we also have $\mb z_i = \mb u_i = \frac{1}{\sigma_i}\mb Y \mb v_i$ where $\mb v_i$ is the $i$-th column of $\mb V$. Hence, for any $\mb a\in \Rbb^d$,
    \begin{equation*}
        |\mb z_i^\top \mb a| =\frac{1}{\sigma_i}|\mb v_i^\top \mb Y^\top \mb a| \leq \frac{\|\mb v_i\|_1}{\sigma_i} \|\mb Y^\top \mb a\|_\infty \leq \frac{d^{1/2+1/(2s)}}{\delta^{s/(s-1)}}\|\mb Y^\top \mb a\|_\infty.
    \end{equation*}
    This ends the proof of the lemma.
\end{proof}
\end{document}

%% file: shortcuts.tex
\newcommand{\trw}{\text{\small TRW}}
\newcommand{\maxcut}{\text{\small MAXCUT}}
\newcommand{\maxcsp}{\text{\small MAXCSP}}
\newcommand{\suol}{\text{SUOL}}
\newcommand{\wuol}{\text{WUOL}}
\newcommand{\crf}{\text{CRF}}
\newcommand{\sual}{\text{SUAL}}
\newcommand{\suil}{\text{SUIL}}
\newcommand{\fs}{\text{FS}}
\newcommand{\fmv}{{\text{FMV}}}
\newcommand{\smv}{{\text{SMV}}}
\newcommand{\wsmv}{{\text{WSMV}}}
\newcommand{\trwp}{\text{\small TRW}^\prime}
\newcommand{\alg}{\text{ALG}}
\newcommand{\rhos}{\rho^\star}
\newcommand{\brhos}{\brho^\star}
\newcommand{\bzero}{{\mathbf 0}}
\newcommand{\bs}{{\mathbf s}}
\newcommand{\bw}{{\mathbf w}}
\newcommand{\bws}{\bw^\star}
\newcommand{\ws}{w^\star}
\newcommand{\Prt}{{\mathsf {Part}}}
\newcommand{\Fs}{F^\star}

\newcommand{\Hs}{{\mathsf H} }

\newcommand{\hL}{\hat{L}}
\newcommand{\hU}{\hat{U}}
\newcommand{\hu}{\hat{u}}

\newcommand{\bu}{{\mathbf u}}
\newcommand{\ubf}{{\mathbf u}}
\newcommand{\hbu}{\hat{\bu}}

\newcommand{\primal}{\textbf{Primal}}
\newcommand{\dual}{\textbf{Dual}}

\newcommand{\Ptree}{{\sf P}^{\text{tree}}}
\newcommand{\bv}{{\mathbf v}}

\newcommand{\bq}{\boldsymbol q}

\newcommand{\rvM}{\text{M}}

\newcommand{\Acal}{\mathcal{A}}
\newcommand{\Bcal}{\mathcal{B}}
\newcommand{\Ccal}{\mathcal{C}}
\newcommand{\Dcal}{\mathcal{D}}
\newcommand{\Ecal}{\mathcal{E}}
\newcommand{\Fcal}{\mathcal{F}}
\newcommand{\Gcal}{\mathcal{G}}
\newcommand{\Hcal}{\mathcal{H}}
\newcommand{\Ical}{\mathcal{I}}
\newcommand{\Lcal}{\mathcal{L}}
\newcommand{\Ncal}{\mathcal{N}}
\newcommand{\Pcal}{\mathcal{P}}
\newcommand{\Scal}{\mathcal{S}}
\newcommand{\Tcal}{\mathcal{T}}
\newcommand{\Ucal}{\mathcal{U}}
\newcommand{\Vcal}{\mathcal{V}}
\newcommand{\Wcal}{\mathcal{W}}
\newcommand{\Xcal}{\mathcal{X}}
\newcommand{\Ycal}{\mathcal{Y}}
\newcommand{\Ocal}{\mathcal{O}}
\newcommand{\Qcal}{\mathcal{Q}}
\newcommand{\Rcal}{\mathcal{R}}

\newcommand{\brho}{\boldsymbol{\rho}}

\newcommand{\Cbb}{\mathbb{C}}
\newcommand{\Ebb}{\mathbb{E}}
\newcommand{\Nbb}{\mathbb{N}}
\newcommand{\Pbb}{\mathbb{P}}
\newcommand{\Qbb}{\mathbb{Q}}
\newcommand{\Rbb}{\mathbb{R}}
\newcommand{\Sbb}{\mathbb{S}}
\newcommand{\Xbb}{\mathbb{X}}
\newcommand{\Ybb}{\mathbb{Y}}
\newcommand{\Zbb}{\mathbb{Z}}

\newcommand{\Rbbp}{\Rbb_+}

\newcommand{\bX}{{\mathbf X}}
\newcommand{\bx}{{\boldsymbol x}}

\newcommand{\btheta}{\boldsymbol{\theta}}

\newcommand{\Pb}{\mathbb{P}}

\newcommand{\hPhi}{\widehat{\Phi}}

\newcommand{\Sigmah}{\widehat{\Sigma}}
\newcommand{\thetah}{\widehat{\theta}}

\newcommand{\indep}{\perp \!\!\! \perp}
\newcommand{\notindep}{\not\!\perp\!\!\!\perp}

\newcommand{\one}{\mathbbm{1}}
\newcommand{\1}{\mathbbm{1}}
\newcommand{\aprx}{\alpha}

\newcommand{\ST}{\Tcal(\Gcal)}
\newcommand{\x}{\mathsf{x}}
\newcommand{\y}{\mathsf{y}}
\newcommand{\Ybf}{\textbf{Y}}
\newcommand{\smiddle}[1]{\;\middle#1\;}

\definecolor{dark_red}{rgb}{0.2,0,0}
\newcommand{\detail}[1]{\textcolor{dark_red}{#1}}

\newcommand{\ds}[1]{{\color{red} #1}}
\newcommand{\rc}[1]{{\color{green} #1}}

\newcommand{\mb}[1]{\ensuremath{\boldsymbol{#1}}}

\newcommand{\metric}{\rho}

\newcommand{\paren}[1]{\left( #1 \right)}
\newcommand{\sqb}[1]{\left[ #1 \right]}

%% file: main.bbl
\begin{thebibliography}{42}
\providecommand{\natexlab}[1]{#1}
\providecommand{\url}[1]{\texttt{#1}}
\expandafter\ifx\csname urlstyle\endcsname\relax
  \providecommand{\doi}[1]{doi: #1}\else
  \providecommand{\doi}{doi: \begingroup \urlstyle{rm}\Url}\fi

\bibitem[Marsden et~al.(2022)Marsden, Sharan, Sidford, and
  Valiant]{marsden2022efficient}
Annie Marsden, Vatsal Sharan, Aaron Sidford, and Gregory Valiant.
\newblock Efficient convex optimization requires superlinear memory.
\newblock In \emph{Conference on Learning Theory}, pages 2390--2430. PMLR,
  2022.

\bibitem[Nemirovskij and Yudin(1983)]{nemirovskij1983problem}
Arkadij~Semenovi{\v{c}} Nemirovskij and David~Borisovich Yudin.
\newblock Problem complexity and method efficiency in optimization.
\newblock 1983.

\bibitem[Yudin and Nemirovskii(1976)]{yudin1976informational}
David~B Yudin and Arkadi~S Nemirovskii.
\newblock Informational complexity and efficient methods for the solution of
  convex extremal problems.
\newblock \emph{Matekon}, 13\penalty0 (2):\penalty0 22--45, 1976.

\bibitem[Shor(1977)]{shor1977cut}
Naum~Z Shor.
\newblock Cut-off method with space extension in convex programming problems.
\newblock \emph{Cybernetics}, 13\penalty0 (1):\penalty0 94--96, 1977.

\bibitem[Tarasov(1988)]{tarasov1988method}
Sergei~Pavlovich Tarasov.
\newblock The method of inscribed ellipsoids.
\newblock In \emph{Soviet Mathematics-Doklady}, volume~37, pages 226--230,
  1988.

\bibitem[Nesterov(1989)]{nesterov1989self}
Ju~E Nesterov.
\newblock Self-concordant functions and polynomial-time methods in convex
  programming.
\newblock \emph{Report, Central Economic and Mathematic Institute, USSR Acad.
  Sci}, 1989.

\bibitem[Atkinson and Vaidya(1995)]{atkinson1995cutting}
David~S Atkinson and Pravin~M Vaidya.
\newblock A cutting plane algorithm for convex programming that uses analytic
  centers.
\newblock \emph{Mathematical programming}, 69\penalty0 (1-3):\penalty0 1--43,
  1995.

\bibitem[Vaidya(1996)]{vaidya1996new}
Pravin~M Vaidya.
\newblock A new algorithm for minimizing convex functions over convex sets.
\newblock \emph{Mathematical programming}, 73\penalty0 (3):\penalty0 291--341,
  1996.

\bibitem[Levin(1965)]{levin1965algorithm}
Anatoly~Yur'evich Levin.
\newblock An algorithm for minimizing convex functions.
\newblock In \emph{Doklady Akademii Nauk}, volume 160, pages 1244--1247.
  Russian Academy of Sciences, 1965.

\bibitem[Bertsimas and Vempala(2004)]{bertsimas2004solving}
Dimitris Bertsimas and Santosh Vempala.
\newblock Solving convex programs by random walks.
\newblock \emph{Journal of the ACM (JACM)}, 51\penalty0 (4):\penalty0 540--556,
  2004.

\bibitem[Lee et~al.(2015)Lee, Sidford, and Wong]{lee2015faster}
Yin~Tat Lee, Aaron Sidford, and Sam Chiu-wai Wong.
\newblock A faster cutting plane method and its implications for combinatorial
  and convex optimization.
\newblock In \emph{2015 IEEE 56th Annual Symposium on Foundations of Computer
  Science}, pages 1049--1065. IEEE, 2015.

\bibitem[Jiang et~al.(2020)Jiang, Lee, Song, and Wong]{jiang2020improved}
Haotian Jiang, Yin~Tat Lee, Zhao Song, and Sam Chiu-wai Wong.
\newblock An improved cutting plane method for convex optimization,
  convex-concave games, and its applications.
\newblock In \emph{Proceedings of the 52nd Annual ACM SIGACT Symposium on
  Theory of Computing}, pages 944--953, 2020.

\bibitem[Anstreicher(2000)]{anstreicher2000volumetric}
Kurt~M Anstreicher.
\newblock The volumetric barrier for semidefinite programming.
\newblock \emph{Mathematics of Operations Research}, 25\penalty0 (3):\penalty0
  365--380, 2000.

\bibitem[McCormick(2005)]{mccormick2005submodular}
S~Thomas McCormick.
\newblock Submodular function minimization.
\newblock \emph{Handbooks in operations research and management science},
  12:\penalty0 321--391, 2005.

\bibitem[Gr{\"o}tschel et~al.(2012)Gr{\"o}tschel, Lov{\'a}sz, and
  Schrijver]{grotschel2012geometric}
Martin Gr{\"o}tschel, L{\'a}szl{\'o} Lov{\'a}sz, and Alexander Schrijver.
\newblock \emph{Geometric algorithms and combinatorial optimization}, volume~2.
\newblock Springer Science \& Business Media, 2012.

\bibitem[Jiang(2021)]{jiang2021minimizing}
Haotian Jiang.
\newblock Minimizing convex functions with integral minimizers.
\newblock In \emph{Proceedings of the 2021 ACM-SIAM Symposium on Discrete
  Algorithms (SODA)}, pages 976--985. SIAM, 2021.

\bibitem[Papadimitriou and Roughgarden(2008)]{papadimitriou2008computing}
Christos~H Papadimitriou and Tim Roughgarden.
\newblock Computing correlated equilibria in multi-player games.
\newblock \emph{Journal of the ACM (JACM)}, 55\penalty0 (3):\penalty0 1--29,
  2008.

\bibitem[Jiang and Leyton-Brown(2011)]{jiang2011polynomial}
Albert~Xin Jiang and Kevin Leyton-Brown.
\newblock Polynomial-time computation of exact correlated equilibrium in
  compact games.
\newblock In \emph{Proceedings of the 12th ACM conference on Electronic
  commerce}, pages 119--126, 2011.

\bibitem[Woodworth and Srebro(2019)]{woodworth2019open}
Blake Woodworth and Nathan Srebro.
\newblock Open problem: The oracle complexity of convex optimization with
  limited memory.
\newblock In \emph{Conference on Learning Theory}, pages 3202--3210. PMLR,
  2019.

\bibitem[Nesterov(2003)]{nesterov2003introductory}
Yurii Nesterov.
\newblock \emph{Introductory lectures on convex optimization: A basic course},
  volume~87.
\newblock Springer Science \& Business Media, 2003.

\bibitem[Steinhardt and Duchi(2015)]{Steinhardt15}
Jacob Steinhardt and John Duchi.
\newblock Minimax rates for memory-bounded sparse linear regression.
\newblock In \emph{Proceedings of The 28th Conference on Learning Theory},
  pages 1564--1587. PMLR, 2015.

\bibitem[Sharan et~al.(2019)Sharan, Sidford, and Valiant]{Sharan2019}
Vatsal Sharan, Aaron Sidford, and Gregory Valiant.
\newblock Memory-sample tradeoffs for linear regression with small error.
\newblock In \emph{Proceedings of the 51st Annual ACM SIGACT Symposium on
  Theory of Computing}, STOC 2019, page 890–901. Association for Computing
  Machinery, 2019.

\bibitem[Mitliagkas et~al.(2013)Mitliagkas, Caramanis, and
  Jain]{Mitliagkas2013}
Ioannis Mitliagkas, Constantine Caramanis, and Prateek Jain.
\newblock Memory limited, streaming pca.
\newblock In \emph{Proceedings of the 26th International Conference on Neural
  Information Processing Systems - Volume 2}, NIPS'13, page 2886–2894, Red
  Hook, NY, USA, 2013. Curran Associates Inc.

\bibitem[Steinhardt et~al.(2016)Steinhardt, Valiant, and Wager]{steinhardt16}
Jacob Steinhardt, Gregory Valiant, and Stefan Wager.
\newblock Memory, communication, and statistical queries.
\newblock In \emph{29th Annual Conference on Learning Theory}, pages
  1490--1516. PMLR, 2016.

\bibitem[Brown et~al.(2021)Brown, Bun, Feldman, Smith, and Talwar]{Brown2021}
Gavin Brown, Mark Bun, Vitaly Feldman, Adam Smith, and Kunal Talwar.
\newblock When is memorization of irrelevant training data necessary for
  high-accuracy learning?
\newblock In \emph{Proceedings of the 53rd Annual ACM SIGACT Symposium on
  Theory of Computing}, STOC 2021, page 123–132. Association for Computing
  Machinery, 2021.

\bibitem[Brown et~al.(2022)Brown, Bun, and Smith]{brown22a}
Gavin Brown, Mark Bun, and Adam Smith.
\newblock Strong memory lower bounds for learning natural models.
\newblock In \emph{Proceedings of Thirty Fifth Conference on Learning Theory},
  pages 4989--5029. PMLR, 2022.

\bibitem[Moshkovitz and Moshkovitz(2017)]{Moshkovitz2017}
Dana Moshkovitz and Michal Moshkovitz.
\newblock Mixing implies lower bounds for space bounded learning.
\newblock In \emph{Proceedings of the 2017 Conference on Learning Theory},
  pages 1516--1566. PMLR, 2017.

\bibitem[Moshkovitz and Moshkovitz(2018)]{Moshkovitz2018}
Dana Moshkovitz and Michal Moshkovitz.
\newblock {Entropy Samplers and Strong Generic Lower Bounds For Space Bounded
  Learning}.
\newblock In \emph{9th Innovations in Theoretical Computer Science Conference
  (ITCS 2018)}, volume~94 of \emph{Leibniz International Proceedings in
  Informatics (LIPIcs)}, pages 28:1--28:20. Schloss Dagstuhl--Leibniz-Zentrum
  fuer Informatik, 2018.

\bibitem[Beame et~al.(2018)Beame, Oveis~Gharan, and Yang]{Beame2018}
Paul Beame, Shayan Oveis~Gharan, and Xin Yang.
\newblock Time-space tradeoffs for learning finite functions from random
  evaluations, with applications to polynomials.
\newblock In \emph{Proceedings of the 31st Conference On Learning Theory},
  pages 843--856. PMLR, 2018.

\bibitem[Garg et~al.(2018)Garg, Raz, and Tal]{Garg2018}
Sumegha Garg, Ran Raz, and Avishay Tal.
\newblock Extractor-based time-space lower bounds for learning.
\newblock In \emph{Proceedings of the 50th Annual ACM SIGACT Symposium on
  Theory of Computing}, STOC 2018, page 990–1002. Association for Computing
  Machinery, 2018.

\bibitem[Kol et~al.(2017)Kol, Raz, and Tal]{Kol2017}
Gillat Kol, Ran Raz, and Avishay Tal.
\newblock Time-space hardness of learning sparse parities.
\newblock In \emph{Proceedings of the 49th Annual ACM SIGACT Symposium on
  Theory of Computing}, STOC 2017, page 1067–1080. Association for Computing
  Machinery, 2017.

\bibitem[Raz(2017)]{Raz2017}
Ran Raz.
\newblock A time-space lower bound for a large class of learning problems.
\newblock In \emph{2017 IEEE 58th Annual Symposium on Foundations of Computer
  Science (FOCS)}, pages 732--742, 2017.
\newblock \doi{10.1109/FOCS.2017.73}.

\bibitem[Nemirovsky et~al.(1983)Nemirovsky, Yudin, and Dawson]{Nemirovsky1983}
A.S. Nemirovsky, D.B. Yudin, and E.R. Dawson.
\newblock \emph{Problem Complexity and Method Efficiency in Optimization}.
\newblock A Wiley-Interscience publication. Wiley, 1983.
\newblock ISBN 978-0471103455.

\bibitem[Woodworth and Srebro(2016)]{Woodworth2016}
Blake~E Woodworth and Nati Srebro.
\newblock Tight complexity bounds for optimizing composite objectives.
\newblock In \emph{Advances in Neural Information Processing Systems},
  volume~29. Curran Associates, Inc., 2016.

\bibitem[Woodworth and Srebro(2017)]{Woodworth2017LowerBF}
Blake~E. Woodworth and Nathan Srebro.
\newblock Lower bound for randomized first order convex optimization.
\newblock \emph{arXiv: Optimization and Control}, 2017.

\bibitem[Nocedal(1980)]{Nocedal1980}
Jorge Nocedal.
\newblock Updating quasi-newton matrices with limited storage.
\newblock \emph{Mathematics of Computation}, 35\penalty0 (151):\penalty0
  773--782, 1980.

\bibitem[Liu and Nocedal(1989)]{liu_limited_1989}
Dong~C. Liu and Jorge Nocedal.
\newblock On the limited memory {BFGS} method for large scale optimization.
\newblock \emph{Mathematical Programming}, 45\penalty0 (1):\penalty0 503--528,
  August 1989.

\bibitem[Lewis and Overton(2013)]{lewis_nonsmooth_2013}
Adrian~S. Lewis and Michael~L. Overton.
\newblock Nonsmooth optimization via quasi-{Newton} methods.
\newblock \emph{Mathematical Programming}, 141\penalty0 (1):\penalty0 135--163,
  October 2013.

\bibitem[Nemirovski(1994)]{nemirovski1994parallel}
Arkadi Nemirovski.
\newblock On parallel complexity of nonsmooth convex optimization.
\newblock \emph{Journal of Complexity}, 10\penalty0 (4):\penalty0 451--463,
  1994.

\bibitem[Balkanski and Singer(2018)]{balkanski2018parallelization}
Eric Balkanski and Yaron Singer.
\newblock Parallelization does not accelerate convex optimization: Adaptivity
  lower bounds for non-smooth convex minimization.
\newblock \emph{arXiv preprint arXiv:1808.03880}, 2018.

\bibitem[Bubeck et~al.(2019)Bubeck, Jiang, Lee, Li, and
  Sidford]{bubeck2019complexity}
S{\'e}bastien Bubeck, Qijia Jiang, Yin-Tat Lee, Yuanzhi Li, and Aaron Sidford.
\newblock Complexity of highly parallel non-smooth convex optimization.
\newblock \emph{Advances in neural information processing systems}, 32, 2019.

\bibitem[Feige and Schechtman(2002)]{feige2002optimality}
Uriel Feige and Gideon Schechtman.
\newblock On the optimality of the random hyperplane rounding technique for max
  cut.
\newblock \emph{Random Structures \& Algorithms}, 20\penalty0 (3):\penalty0
  403--440, 2002.

\end{thebibliography}
